\documentclass{article}

\usepackage[preprint]{neurips_2026}

\usepackage[utf8]{inputenc} 
\usepackage[T1]{fontenc}    
\usepackage{hyperref}       
\usepackage{url}            
\usepackage{booktabs}       
\usepackage{amsfonts}       
\usepackage{nicefrac}       
\usepackage{microtype}      
\usepackage{xcolor}         

\usepackage{graphicx}
\usepackage{wrapfig}        
\usepackage{float}          
\usepackage{placeins}       
\graphicspath{{./}}
\usepackage{subcaption}
\usepackage{amsmath,amssymb}
\usepackage{xspace}
\usepackage{mathtools}
\usepackage{amsthm}

\IfFileExists{dsfont.sty}{\usepackage{dsfont}}{}
\IfFileExists{upgreek.sty}{\usepackage{upgreek}}{}
\providecommand{\mathds}[1]{\ensuremath{\mathchoice%
    {\hbox{$\mathsurround=0pt\rm#1\!\!#1$}}%
    {\hbox{$\mathsurround=0pt\rm#1\!\!#1$}}%
    {\hbox{$\mathsurround=0pt\scriptstyle\rm#1\!\!#1$}}%
    {\hbox{$\mathsurround=0pt\scriptscriptstyle\rm#1\!\!#1$}}}}
\providecommand{\uppi}{\pi}

\usepackage[capitalize,noabbrev]{cleveref}

\usepackage{algorithm}
\usepackage{algorithmic}

\theoremstyle{plain}
\newtheorem{theorem}{Theorem}[section]
\newtheorem{proposition}[theorem]{Proposition}
\newtheorem{lemma}[theorem]{Lemma}

\theoremstyle{definition}
\newtheorem{definition}[theorem]{Definition}

\theoremstyle{remark}

\title{Knowledge-Free Correlated Agreement for Incentivizing Federated Learning}

\author{%
  Leon Witt \\
  SIMIS Shanghai \\
  Shanghai, China \\
  \texttt{leon.witt@simis.cn} \\
  \And
  Togrul Abbasli \\
  Tsinghua University \\
  Beijing, China \\
  \texttt{tgl22@mails.tsinghua.edu.cn} \\
  \And
  Kentaroh Toyoda \\
  Keio University \\
  Yokohama, Japan \\
  \texttt{toyoda@sasase.ics.keio.ac.jp} \\
  \And
  Wojciech Samek \\
  Technical University of Berlin (TU Berlin) \\
  Berlin, Germany \\
  \texttt{wojciech.samek@hhi.fraunhofer.de} \\
  \And
  Lucy Klinger \\
  SIMIS Shanghai \\
  Shanghai, China \\
  \texttt{lucy.klinger@simis.cn} \\
}

\begin{document}

\maketitle

\begin{abstract}
We introduce Knowledge-Free Correlated Agreement (KFCA) to reward client contributions in federated learning (FL) without relying on ground truth, a public test set, or distribution knowledge. Under categorical reports and an honest majority, KFCA is strictly truthful, addressing Correlated Agreement's (CA) label-flipping vulnerability. We evaluate KFCA on federated LLM adapter tuning and a real-world PCB inspection task, showing efficient, real-time reward computation suitable for decentralized and blockchain-based incentive designs.
\end{abstract}

\section{Introduction}
Federated learning (FL) enables multiple clients to collaboratively train a model without sharing raw data \cite{BrendanMcMahan2017}. For FL to become mainstream, a key unsolved issue is how to assess and incentivize client contributions: updates are hard to verify, yet the reward rule must discourage low-effort participation and malicious reporting \cite{SP_MD_ARXIV, SP_MD_IEEE, SP_FL_TECHRXIV, SP_FL_MD, LeonSP}. Classical contribution measures such as Shapley value require a verifiable utility (often via a public test set) and can be computationally expensive, which becomes problematic in settings without clear ground truth. A comprehensive overview of contribution measurement techniques in FL can be found in Appendix~\ref{app:background}.

Multi-task peer prediction (MTPP) mechanisms avoid explicit ground truth by rewarding agreement patterns across clients. A prominent example is Correlated Agreement (CA) \cite{OriginalCA}, but CA has two practical drawbacks for FL: (i) it requires the server to estimate the full report-correlation matrix $\Delta$ from \emph{all} clients' joint reports---incompatible with decentralized FL and adding per-round overhead that is quadratic in the number of clients and linear in the number of reports per client---and (ii) it admits profitable informed attacks such as coordinated label flipping, in which a client inverts their labels and still receives full reward.

We introduce Knowledge-Free Correlated Agreement (KFCA), a simple scoring rule for MTPP based on Dasgupta and Ghosh \cite{DasguptaGhosh2013}. Under a categorical-world condition (common in single-truth classification and Neural Network parametrization, enforceable via preprocessing), KFCA is strongly truthful, avoids CA's label-flipping equilibrium, and supports real-time reward computation without aggregating all reports. This is especially relevant for federated Large Language Model (LLM) adapter fine-tuning (e.g., low-rank adaptation (LoRA) \cite{hu2022lora} / weight-decomposed low-rank adaptation (DoRA) \cite{liu2024dora} with federated averaging (FedAvg) \cite{BrendanMcMahan2017}), where contributions are hard to evaluate ex post and a ``ground-truth'' contribution score is typically unavailable. Specifically, the contributions of this work are:




\begin{enumerate}
\item A multi-task peer-prediction reward mechanism for federated learning that is knowledge-free (no global $\Delta$ estimation or report aggregation), applicable across a wide range of FL tasks, and capable of real-time reward computation---making it suitable for blockchain-based decentralized incentivized FL.

\item Under a categorical-world condition on the induced correlation structure---which we show holds in most FL contexts, including settings with heavily non-IID data distributions (Appendix~\ref{app:non_iid_categorical})---KFCA is strongly truthful: as long as fewer than $50\%$ of clients deviate, honest, effortful reporting strictly maximises expected reward, eliminating CA's label-flipping vulnerability and matching the standard Byzantine-fault-tolerance threshold for any peer-prediction mechanism without ground truth (Appendix~\ref{app:permutation_indistinguishability}).

\item Empirically, KFCA achieves orders-of-magnitude lower reward-computation cost than Shapley-value estimators and can be applied where no ground truth is available. We demonstrate it's feasibility on state-of-the-art federated LLM adapter tuning with LoRA/DoRA and real-world Printed Circuit Board (PCB) quality inspection.
\end{enumerate}
\section{System Setting}

In FL, the challenge is to incentivize clients to participate, given that their contributions are not easily verifiable. Unlike methods such as the Shapley value, which determine an explicit ground-truth contribution by verifying each contribution given a performance metric (e.g., accuracy) on a test set, MTPP aims to elicit latent information from multiple clients through a scoring method based on correlations of clients' reports without the need to know the ground-truth. In this model, clients are assigned multiple tasks. An example of a task would be to classify an image or fine-tune an LLM on a domain-specific data shard, where a client has to invest \textit{effort} (e.g., training a classifier/LLM) to receive a \textit{signal} (e.g., class of the image). An optimal MTPP mechanism should incentivize clients (i) to train their AI model and then (ii) honestly communicate the signals, leading to a Bayesian Nash equilibrium in which the aforementioned behavior is most rewarded, ensuring no other strategy is more lucrative.

Formally, consider an FL system with clients \(N=\{1,\dots,n\}\) collaborating on tasks \(M=\{1,\dots,m\}\). Each task \(k\in M\) has an unobserved latent truth \(Y^k\in[L]=\{1,\dots,L\}\) (e.g., a correct class label or optimal model update), drawn independently and identically from a categorical prior distribution over \([L]\). All clients assigned to task \(k\) share the same latent truth \(Y^k\) but cannot observe it directly. The mechanism designer's goal is to construct a payment rule that incentivizes each client to exert effort to obtain informative signals correlated with \(Y^k\) and to report these signals truthfully, so that truthful, effortful reporting forms a Bayesian Nash equilibrium.

Figure~\ref{fig:MTPP_algorithm} visualizes MTPP, and Table~\ref{tab:fl_mapping} maps it to the two FL instantiations evaluated here: KFCA-D (clients scored against a shared public dataset) and KFCA-QP (scored against the signs of their quantized parameter updates, when no public dataset is available); illustrated in Figure~\ref{fig:public_test} (Section~\ref{KFCA}).

\begin{figure}[t]
\centering
\begin{minipage}[c]{0.5\textwidth}
    \centering
    \includegraphics[width=\linewidth]{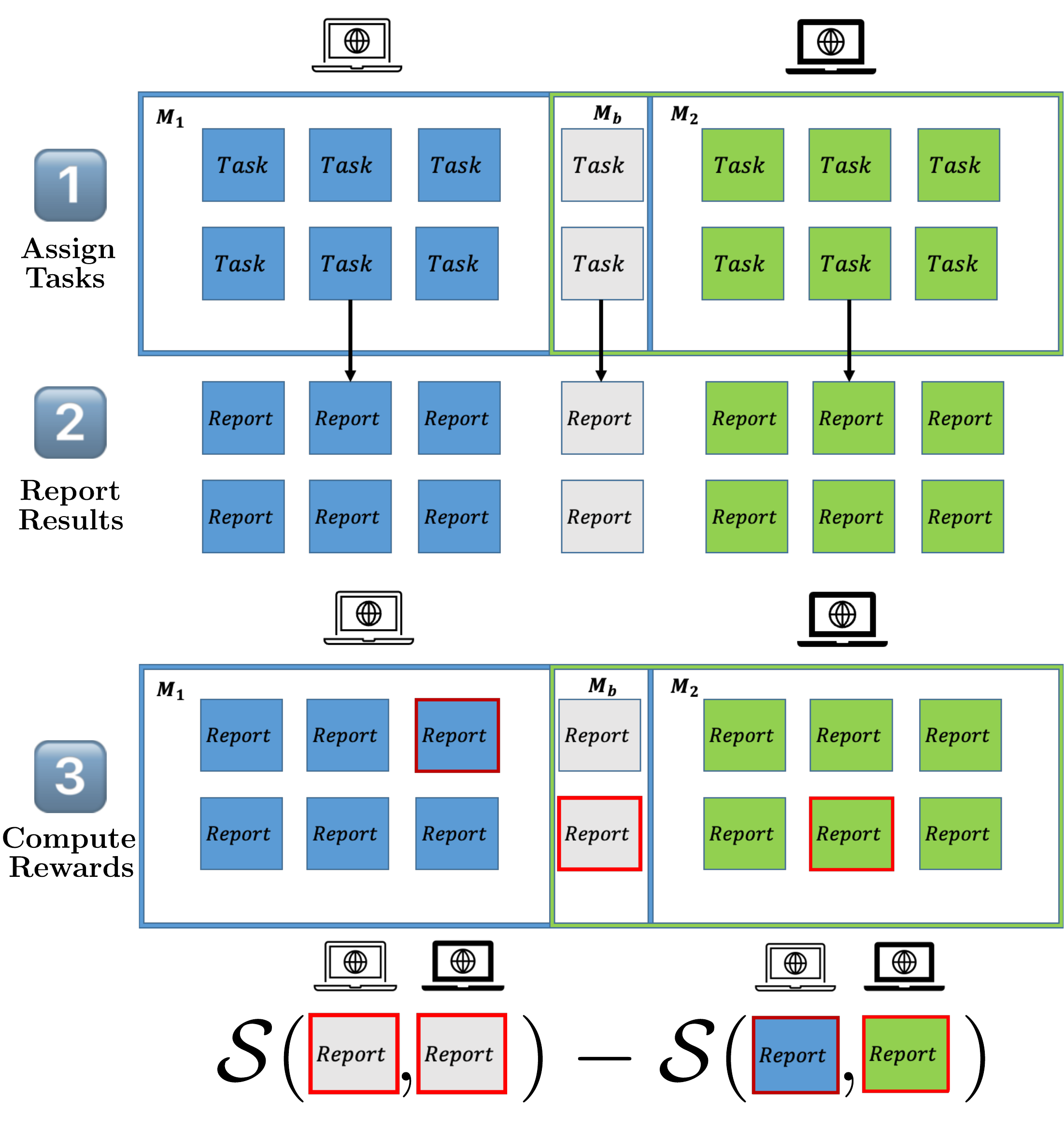}
    \captionof{figure}{Illustration of the multi-task peer-prediction mechanism. }
    \label{fig:MTPP_algorithm}
\end{minipage}\hfill
\begin{minipage}[c]{0.45\textwidth}
    \centering
    \small
    \setlength{\tabcolsep}{4pt}
    \renewcommand{\arraystretch}{1.2}
    \begin{tabular}{@{}p{2cm}p{4.2cm}@{}}
        \toprule
        \textbf{MTPP} & \textbf{ FL (KFCA-D / KFCA-QP)} \\
        \midrule
        Task $k$ & Classify a test-set sample / decide update sign at parameter $p$ of a model update $\Delta\theta$  \\
        Latent truth $Y^k$ & Correct class label / optimal update direction at $p$ \\
        Signal $Z_i^k$ & Local classification of an image / $\mathrm{sign}(\Delta\theta_i[p]) \in \{-1,+1\}$ \\
        Effort $e_i^k\!=\!1$ & Train local model / fine-tune LLM adapter (LoRA/DoRA) \\
        No effort $e_i^k\!=\!0$ & Random prediction / zero or random update \\
        Report $R_i^k$ & Submitted label / submitted sign \\
        \bottomrule
    \end{tabular}
    \captionof{table}{MTPP mapped to examples of KFCA-D (public test set) and KFCA-QP (sign-quantized parameter updates) instantiations.}
    \label{tab:fl_mapping}
\end{minipage}
\end{figure}

\subsection{Preliminary}

\paragraph{Assumptions:}
We assume:
(i) payments can be made to clients;
(ii) clients are risk-neutral, meaning their utility equals expected payment minus effort cost;  
(iii) effort is binary: \( e_i^k \in \{0,1\} \), where \( e_i^k = 1 \) indicates the agent exerts effort and receives an informative signal, and \( e_i^k = 0 \) corresponds to no effort; and
(iv) tasks are ex-ante identical: drawn independently from a common prior distribution over latent truths.
See Appendix~\ref{app:MTTPassumptions} for formal definitions.

\paragraph{Reporting Strategy.}
After (possibly) exerting effort, client $i$ submits a report $R_i^k \in [L]$ that may be a (randomized) function of its private signal $Z_i^k$, inducing the chain $Y^k \to Z_i^k \to R_i^k$. Formally a strategy is a conditional distribution $F_i(r\mid a) := \mathbb{P}(R_i^k = r \mid Z_i^k = a)$ with $\sum_r F_i(r\mid a)=1$. Because expected payments are linear in $F_i$, every randomized strategy's expected reward is a convex combination of deterministic rewards; we therefore restrict attention to deterministic strategies $F_i(r\mid a) = \mathds{1}\{r = f_i(a)\}$ for some $f_i : [L]\to [L]$. The standard \emph{informed} vs.\ \emph{uninformed} classification of $f_i$ is recalled in Appendix~\ref{app:MTTPassumptions}.

\paragraph{Truthful Strategy.}
The strategy of interest is the one in which a client reports its private signal as-is. We write it as:
\[
f^\star(a) = a, \quad \forall a \in [L].
\]
The formal optimality criterion (\emph{informed truthfulness}) is stated in Definition~\ref{def:informed_truthful} once the expected reward $E$ has been introduced.

\subsubsection{Multi-task Peer Prediction}

Consider two clients assigned to a set of tasks indexed by \( k \in [M] = \{1,\dots,m\} \), with \( m \ge 3 \).  
Let \([L]\) denote the discrete label space.  
We partition the task set \(M\) into three disjoint subsets:  
a set of \emph{bonus tasks} \( M_b \subset M \),  
and two disjoint \emph{penalty task sets} \( M_1, M_2 \subset M \setminus M_b \) for clients 1 and 2, respectively.
While our analysis focuses on client pairs, the mechanism extends naturally to multiple clients via random pairing on shared tasks.

The MTPP mechanism is the workhorse on which CA and KFCA both build. It defines how, given any score matrix, clients are paired and rewarded so that the resulting expected reward $E$ has a clean algebraic form (Eq.~\ref{eq:expected_reward_delta}).
\begin{definition}[Multi-task Peer Prediction (MTPP) Mechanism]
\label{def:MTPP}
Given a score matrix \( \mathcal{S} : [L] \times [L] \to \mathbb{R} \), the mechanism proceeds as follows:
\begin{enumerate}
    \item Assign each client at least two tasks, with at least one task in common with a peer.
    \item For each shared task \(k\), let clients report \( r_1^k \), \( r_2^k \).  
          Randomly select a subset \( M_b \) of bonus tasks and partition the remaining tasks into two disjoint sets \( M_1 \) and \( M_2 \).
    \item For each \( k \in M_b \), draw \( p_1 \in M_1 \), \( p_2 \in M_2 \) uniformly at random.  
          The payment for task \(k\) is:
          \[
          \mathcal{S}(r_1^k, r_2^k) - \mathcal{S}(r_1^{p_1}, r_2^{p_2}).
          \]
          Each client’s total payment is the sum over all their assigned bonus tasks.
\end{enumerate}

\paragraph{Expected Reward.}
Let \( F_1(r_1 \mid a) \) and \( F_2(r_2 \mid b) \) be (possibly randomized) reporting strategies, given private signals \( Z_1 = a \), \( Z_2 = b \).  
Assuming independent penalty-task signals, the expected reward under MTPP becomes:
\begin{equation}
\label{eq:expected_reward_delta}
\begin{aligned}
E(F_1, F_2)
&= \sum_{a,b \in [L]} \Delta(a,b) &\quad \cdot \sum_{r_1,r_2 \in [L]} F_1(r_1 \mid a)\, F_2(r_2 \mid b)\, \mathcal{S}(r_1, r_2)
\end{aligned}
\end{equation}
where \( \Delta(a,b) \) captures the correlation structure of signal pairs beyond chance.
\end{definition}

The expected reward decomposes through a single object — the \emph{delta matrix} — which encodes how often two clients' signals co-occur beyond what independence would predict. Every truthfulness statement in this paper is ultimately a statement about the sign pattern of $\Delta$.
\begin{definition}[Delta Matrix]
\label{def:delta}
The \emph{delta matrix} \( \Delta \in \mathbb{R}^{L \times L} \) quantifies excess correlation in client signal reports:
\[
\Delta(a,b)
:= \mathbb{P}(Z_1 = a, Z_2 = b)
- \mathbb{P}(Z_1 = a)\, \mathbb{P}(Z_2 = b),
\]
for \( a,b \in [L] \), ~where~
$\sum_b \Delta(a,b) = \sum_a \Delta(a,b) = 0$.
\end{definition}

With $E$ now defined (Eq.~\eqref{eq:expected_reward_delta}), we can state the optimality criterion that any reasonable peer-prediction mechanism should satisfy: truthful reporting must be at least as rewarding as any other strategy.
\begin{definition}[Informed Truthfulness]
\label{def:informed_truthful}
A peer-prediction mechanism is \emph{informed-truthful} if
\[
E(f^\star, f^\star) \ge E(f_1, f_2)
\]
for all reporting functions \( f_1, f_2 : [L] \to [L] \), where \( f^\star(a) = a \) is the truthful reporting function.
\end{definition}

Using the delta matrix, \citet{OriginalCA} extended the mechanism of \citet{DasguptaGhosh2013} to the multi-task setting, introducing the \emph{Correlated Agreement (CA)} mechanism.

CA is the canonical knowledge-heavy benchmark we will compare against: it requires the server to know $\Delta$, and rewards report pairs whose joint frequency is above-chance.
\begin{definition}[Correlated Agreement (CA)]
\label{def:CA}
The \emph{Correlated Agreement} mechanism defines the score matrix \( \mathcal{S}_{\mathrm{CA}} : [L] \times [L] \to \{0,1\} \) based on delta matrix \( \Delta \) as
\[
\mathcal{S}_{\mathrm{CA}}(a, b) := \operatorname{sign}(\Delta(a, b))= 
\begin{cases}
1 & \text{if } \Delta(a, b) > 0, \\
0 & \text{otherwise}.
\end{cases}
\]
\end{definition}

\begin{theorem}[Informed Truthfulness of CA]
\label{thm:CA_truthful}
Suppose clients are homogeneous, tasks are ex-ante identical, and clients cannot distinguish between bonus and penalty tasks. Then  Correlated Agreement mechanism is \emph{informed-truthful}.
\end{theorem}

The proof is provided in Appendix~\ref{app:truthfulproof}.

\subsection{Limitations of CA}
\label{sec:ca_limitations}

CA has two practical drawbacks that motivate KFCA. First, computing the score matrix requires the server to estimate $\Delta$ from \emph{all} clients' joint reports---an $\mathcal{O}(n^2(m+L^2))$ per-round operation that forces centralization and rules out streaming or on-chain payouts (Appendix~\ref{app:ca_knowledge}). Second, CA scores depend only on the sign pattern of $\Delta$, so a coordinated label permutation preserves all above-chance correlations and earns the same reward as truthful reporting. Concretely, in the binary case
\[
E_{\mathrm{CA}}(f^\star, f^\star) \;=\; E_{\mathrm{CA}}(f_{\mathrm{flip}}, f_{\mathrm{flip}}) \;=\; 0.5,
\]
i.e., truthful and label-flipped reporting are indistinguishable under CA (worked derivation in Appendix~\ref{app:ca_labelflip_example}).

\begin{figure}[t]
\centering
\begin{subfigure}[t]{0.42\linewidth}
  \centering
  \includegraphics[width=\linewidth]{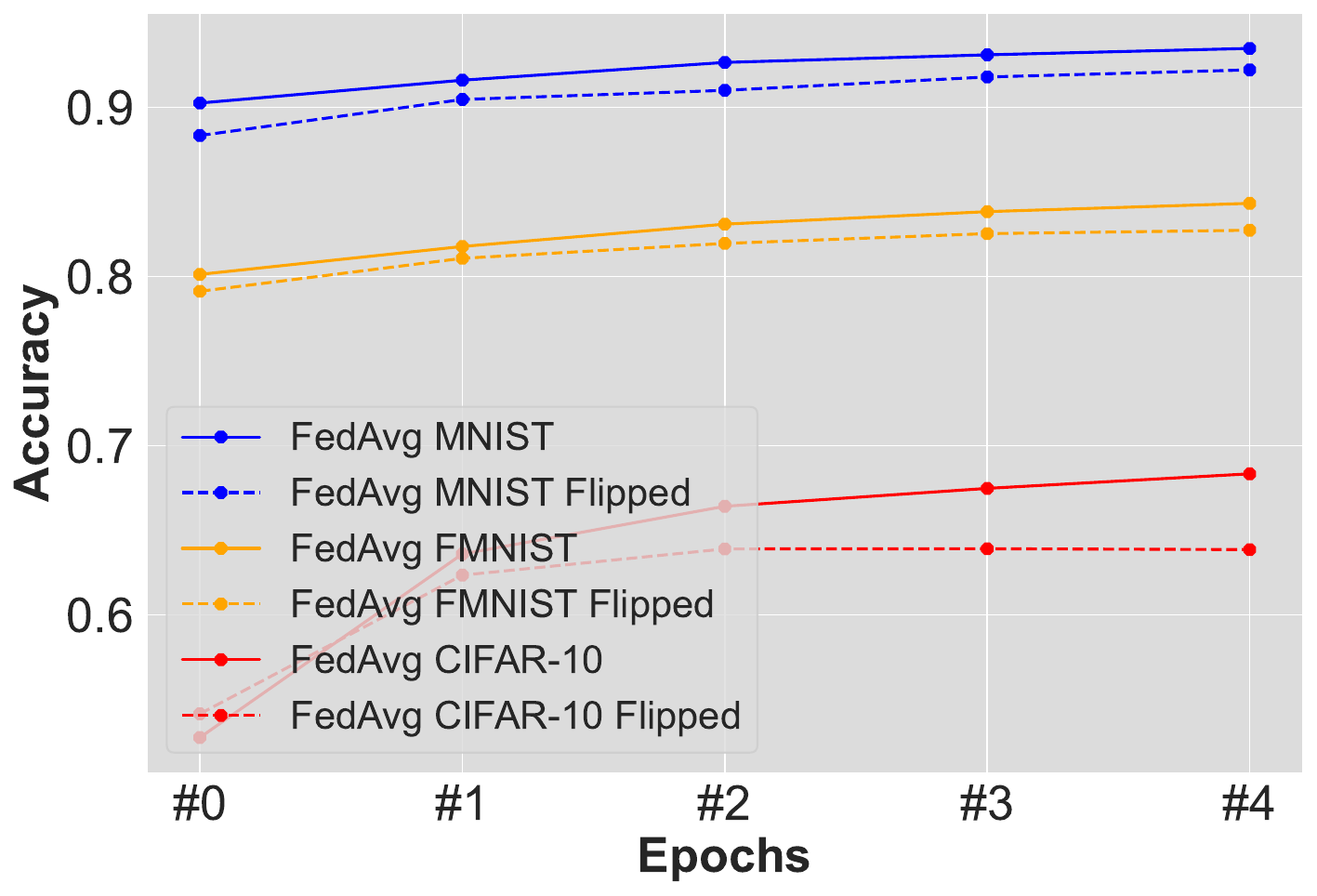}
  \caption{Global accuracy}
  \label{fig:noisy_acc}
\end{subfigure}\hspace{2em}
\begin{subfigure}[t]{0.42\linewidth}
  \centering
  \includegraphics[width=\linewidth]{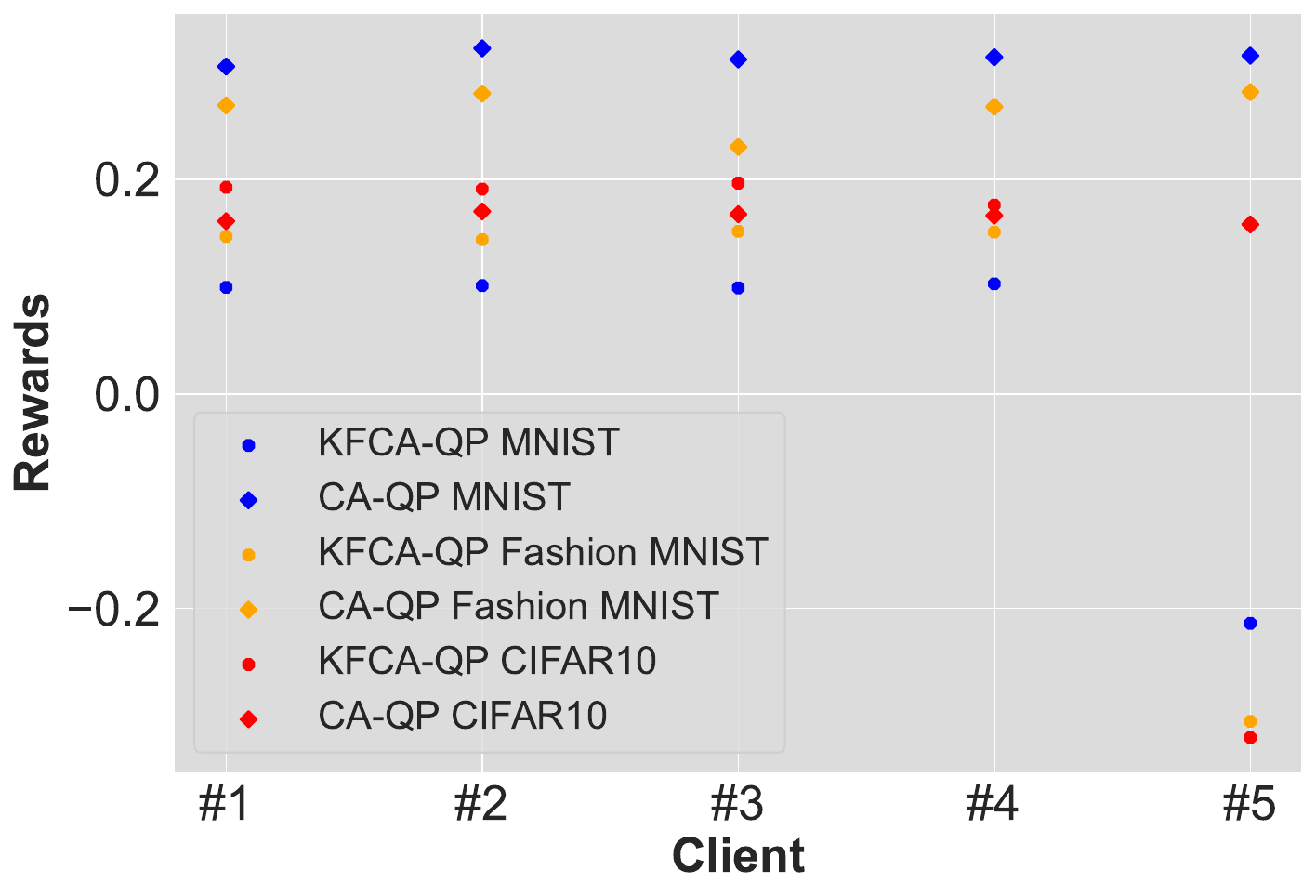}
  \caption{Rewards}
  \label{fig:noisy_acc_ds}
\end{subfigure}
\caption{KFCA-QP vs.\ CA-QP under a sign-flip attack: (a) global test accuracy over FL rounds, (b) per-client rewards. CA-QP fails to penalize the malicious client (\#5); KFCA-QP correctly assigns it a lower reward.}
\label{fig:labelflipping}
\end{figure}

Figure~\ref{fig:labelflipping} confirms this empirically: a single malicious client performing coordinated label flipping degrades the global model while still receiving full reward under CA. The underlying issue is that CA rewards statistical correlation rather than semantic correctness, so any attacker who preserves marginal distributions while disrupting truthful alignment escapes detection. To address this, we seek a mechanism $\mathcal{M}_{\mathrm{KFCA}}$ that is \emph{knowledge-free} (does not require estimating $\Delta$) and \emph{strictly truthful}, meaning $E(f^\star, f^\star) > E(f_1, f_2)$ for all $f_1, f_2$.

We next strengthen CA's weak truthfulness to \emph{strict} truthfulness under an additional sign-pattern condition on $\Delta$.
\section{Knowledge-Free Correlated Agreement}
\label{KFCA}

\begin{figure}[H]
    \centering
    \includegraphics[width=\linewidth]{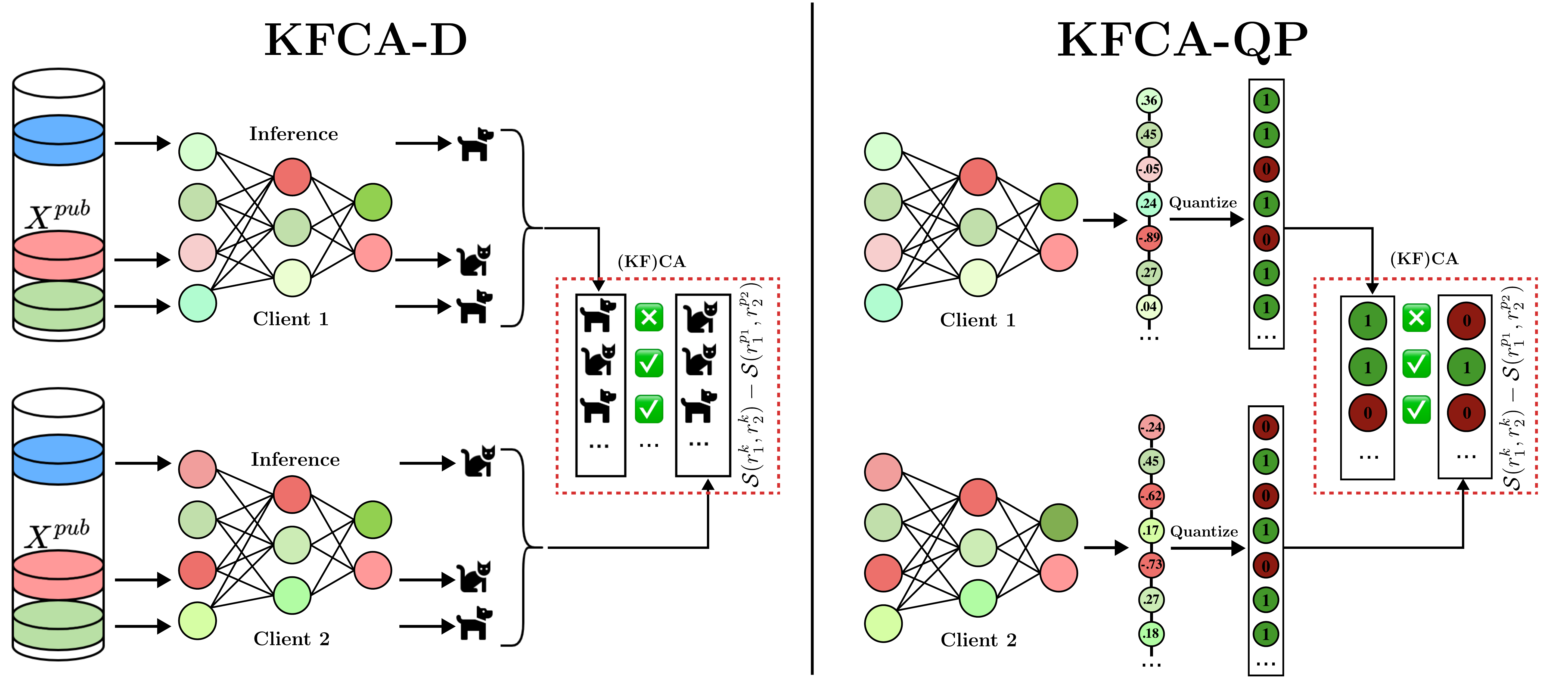}
    \caption{KFCA on a public test set (KFCA-D) and on quantized parameter updates (KFCA-QP).}
    \label{fig:public_test}
\end{figure}

We propose the \emph{Knowledge-Free Correlated Agreement} (KFCA) mechanism \( \mathcal{M}_{\mathrm{KFCA}} \), a variant of the Correlated Agreement (CA) method~\cite{DasguptaGhosh2013} tailored for FL settings where the central server has no access to clients’ signal distributions. KFCA removes the need to estimate the signal correlation matrix and instead leverages a structural assumption on signal alignment to ensure strong incentive guarantees.

We consider a classification setup where each task \( k \in M \) has a latent label \( Y^k \in [L] \), and each client \( i \) receives a private signal \( Z_i^k \in [L] \) drawn from:
\[
\mathbb{P}(Z_i^k = a \mid Y^k = y) =
\begin{cases}
P_i(a \mid y), & \text{if } e_i^k = 1, \\
Q_i(a), & \text{if } e_i^k = 0,
\end{cases}
\]
where \( e_i^k \in \{0,1\} \) indicates whether the client exerted effort. When effort is exerted, \( P_i(\cdot \mid y) \) is an informative, label-dependent distribution; otherwise, \( Q_i \) is an uninformative baseline independent of \( y \).Under effort, we assume signals are conditionally independent given the ground truth:
\[
\mathbb{P}(Z_1^k = a, Z_2^k = b \mid Y^k = y) = P_1(a \mid y)\, P_2(b \mid y).
\]
The marginal (unconditional) distribution of signals is then:
\[
\mathbb{P}(Z_1 = a, Z_2 = b) = \sum_{y} \mathbb{P}(Y = y)\, P_1(a \mid y)\, P_2(b \mid y).
\]
In many classification tasks with meaningful label semantics (e.g., digit or object recognition), signal distributions are structured such that clients are more likely to agree when their signals are accurate. This motivates the following assumption on the joint correlation structure.

\begin{definition}[Categorical-World Condition]
\label{def:categorical}
The environment satisfies the \emph{categorical-world condition} if the signal correlation matrix \( \Delta \in \mathbb{R}^{L \times L} \), defined as
\[
\Delta(a,b)
= \mathbb{P}(Z_1 = a, Z_2 = b) - \mathbb{P}(Z_1 = a)\, \mathbb{P}(Z_2 = b),
\]
satisfies the sign pattern:
\[
\Delta(a,a) > 0, \quad \Delta(a,b) < 0 \quad \text{for all } a \ne b.
\]
By construction, we have
$\sum_b \Delta(a,b) = \sum_a \Delta(a,b) = 0$.
\end{definition}
This condition captures environments where signals are accurate and conditionally independent: matching labels are positively correlated, while mismatches are anti-correlated. It rules out systematic confusion between distinct labels and ensures that agreement reflects genuine alignment on the latent truth.
KFCA leverages this structure to reward label agreement directly, without estimating \( \Delta \). Under the categorical-world condition, agreement implies information, enabling KFCA to retain the truthfulness guarantees of CA in a decentralized, knowledge-free setting.

\paragraph{Realism under non-IID client data.}
Non-IID and the categorical-world condition act on different layers: non-IID specifies the \emph{training-data} partition, while the categorical-world condition is a sign condition on the \emph{post-training} joint signal distribution. In the binary case ($L{=}2$) it reduces to a per-client inequality $\alpha_i := \mathbb{P}(\widetilde{Z}_i \ne Y) < \tfrac{1}{2}$---each client's quantized update need only be \emph{better than random}, easy to satisfy since clients descend the same loss from the same checkpoint; for general $L$, non-IID rescales $|\Delta|$ but preserves $\mathrm{sign}(\Delta)$. A per-scenario mapping and a Dirichlet sweep (concentration $\alpha_{\mathrm{dir}}\in\{0.1,\dots,100\}$, severe label skew to near-IID) verifying the condition on MNIST and AG~News are in Appendix~\ref{app:non_iid_categorical}.

\begin{definition}[Knowledge-Free Correlated Agreement]
\label{def:KFCA}
Under the categorical-world condition, the \emph{Knowledge-Free Correlated Agreement} (KFCA) mechanism uses scoring rule
\[
\mathcal{S}_{\mathrm{KFCA}}(r_1, r_2) := \mathds{1}\{r_1 = r_2\},
\]
which rewards clients only when their reports match. KFCA is a special case of the multi-task peer-prediction mechanism (Definition~\ref{def:MTPP}) instantiated with this score.
\end{definition}

\paragraph{Expected Reward.}
Let \( f_1, f_2 : [L] \to [L] \) be deterministic reporting strategies. Substituting the KFCA score into the MTPP expected reward formula (Eq.~\eqref{eq:expected_reward_delta}) gives:
\begin{equation}
\label{eq:expected_reward_KFCA}
E(f_1, f_2) = \sum_{a,b \in [L]} \Delta(a, b)\, \mathds{1}\{f_1(a) = f_2(b)\}.
\end{equation}
This measures the total correlation between signals whose mapped reports agree under strategies \( f_1 \) and \( f_2 \).

\begin{proposition}[Strict Truthfulness under Categorical Reports]
\label{prop:strong_truthfulness}
Suppose the delta matrix \( \Delta \in \mathbb{R}^{L \times L} \) satisfies the categorical-world condition (Definition~\ref{def:categorical}). Then, under the KFCA mechanism, the truthful reporting strategy \( f^\star(a) = a \) achieves the highest expected reward among all deterministic strategies, $E(f^\star, f^\star) = \max_{f_1, f_2} E(f_1, f_2)$, with equality only when $f_1 = f_2 = \pi$ for some bijection $\pi : [L] \to [L]$; this shared-permutation residual is removed under the honest-majority assumption (Appendix~\ref{app:permutation_indistinguishability}, Proposition~\ref{prop:robustness}).
\end{proposition}

\begin{proof}
For each report label \( r \in [L] \), define pre-image sets:
\[
A_r := \{ a \in [L] : f_1(a) = r \}, \quad
B_r := \{ b \in [L] : f_2(b) = r \}.
\]
By Equation~\eqref{eq:expected_reward_KFCA}, the expected reward becomes:
\[E(f_1, f_2) = \sum_{r \in [L]} \sum_{a \in A_r} \sum_{b \in B_r} \Delta(a, b).\]

Under the categorical-world condition, \( \Delta(a,a) > 0 \) and \( \Delta(a,b) < 0 \) for all \( a \ne b \). Hence, to maximize the reward, the strategy must:
(i). Include all positive diagonal terms \( \Delta(a,a) \),
(ii). Exclude all negative off-diagonal terms \( \Delta(a,b) \) with \( a \ne b \).
This is only possible if for each \( r \), the sets \( A_r \) and \( B_r \) are either empty or both equal to \( \{a\} \) for some \( a \). In other words, \( f_1 = f_2 = \sigma \) for some bijection \( \sigma: [L] \to [L] \). Any deviation introduces off-diagonal penalties or omits diagonal rewards, leading to strictly lower total.
Therefore, the truthful strategy (or any shared permutation) uniquely maximizes expected reward.
\end{proof}
\begin{proposition}[Robustness to Malicious Clients]
\label{prop:robustness}
In the binary label setting with uniform class prior and symmetric noise, the KFCA mechanism tolerates any fraction less than \(50\%\) of malicious clients, while still strictly incentivizing truthful reporting for the honest ones.
\end{proposition} 
The proof is provided in Appendix~\ref{app:robustnessproof}.

\subsection{Advantages of KFCA}
Under the categorical-world model, KFCA offers three key advantages. First, it eliminates the need to estimate the delta matrix, removing the dependence on full report access and improving privacy. Second, it enables efficient real-time reward computation, which is crucial for large models and applications such as smart contracts \cite{LeonSP}. Third, KFCA enforces strong truthfulness, avoiding the label-flipping vulnerabilities of CA and ensuring incentive-aligned participation.

\subsection{Beyond Ideal Categorical Cases}
\label{sec:kfca-fl}
While KFCA is provably truthful under the categorical-world condition (Definition~\ref{def:categorical}), real-world federated learning (FL) sometimes involves continuous, high-dimensional, or heterogeneous signals. These complexities may violate the ideal sign structure of the empirical correlation matrix \( \Delta \), making direct application of KFCA infeasible.
To address this, we introduce a transformation pipeline that maps arbitrary client signals into categorical reports aligned with a shared latent truth. Under the assumptions that (i) signals are conditionally independent given the truth and (ii) each signal is informative, there exists a mapping
$Z_i \mapsto \widetilde{Z}_i \in [L']$
such that the induced matrix \( \widetilde{\Delta} \) satisfies
\[
\operatorname{sign}(\widetilde{\Delta}(a,b)) =
\begin{cases}
> 0 & \text{if } a = b, \\
< 0 & \text{if } a \ne b.
\end{cases}
\]
This restores the conditions under which KFCA remains truthful and knowledge-free, even when the original signals are not categorical.

Concrete instantiations cover classification (relabel via \emph{maximum a posteriori} (MAP) inference) and optimization-based FL (sign-quantize parameter updates); the underlying three-stage pipeline---signal-space transformation, latent-truth alignment, and categorical regularization---generalizes beyond these. Full details, including the per-scenario non-IID analysis, are in Appendices~\ref{app:categoricalpipeline} and~\ref{app:non_iid_categorical}.

\section{Experiments}
We evaluate KFCA along two axes: (i) \emph{efficiency} (scalability and runtime) and (ii) \emph{reward fidelity} (agreement with exact Shapley value). We then demonstrate applicability in settings where Shapley-style ground truth is unavailable: a real-world PCB inspection task and federated LLM adapter tuning. Implementation details and extended setups are deferred to Appendix~\ref{app:experiments_details}.

\subsection{Scalability of KFCA}
KFCA's per-round cost is $\mathcal{O}(n p m)$---linear in clients $n$ for a fixed number of peers per client $p$ ($m$ = report length: test-set size for KFCA-D, parameter count for KFCA-QP), since each pair is a sign-comparison sweep. CA, by contrast, costs $\mathcal{O}(n^2(m+L^2))$ (Section~\ref{sec:ca_limitations}): quadratic in $n$ from re-estimating $\hat{\Delta}$ across all $\binom{n}{2}$ pairs. Fewer peers give faster computation at the cost of standard error $\sigma/\sqrt{p}$ ($\sigma$: stdev of comparison scores). Wall-clock measurements (Apple~Silicon M2~Pro) are in Appendix~\ref{app:scalability}, Fig.~\ref{fig:time}.

\subsection{KFCA vs. Shapley value: Experimental setup}

Due to the heavy computation of the original Shapley value (Appendix~\ref{app:shapley_value}), we train a simple CNN model that has a total of 21,840 parameters on the MNIST dataset \cite{MH_MNIST}. We split the dataset into five equal silos. The KFCA algorithms are evaluated under five different i.i.d. and non-i.i.d. FL cases with 10 clients, similar to \cite{wei2020efficient}:
We consider five standard FL heterogeneity/noise settings (i.i.d., label skew, size skew, label noise, feature noise); full case definitions are in Appendix~\ref{app:experiments_details}.

\paragraph{Variance \& Computational Overhead.}

\begin{figure}[!htbp]
    \centering
    \begin{subfigure}[c]{0.49\textwidth}
        \centering
        \includegraphics[width=\linewidth]{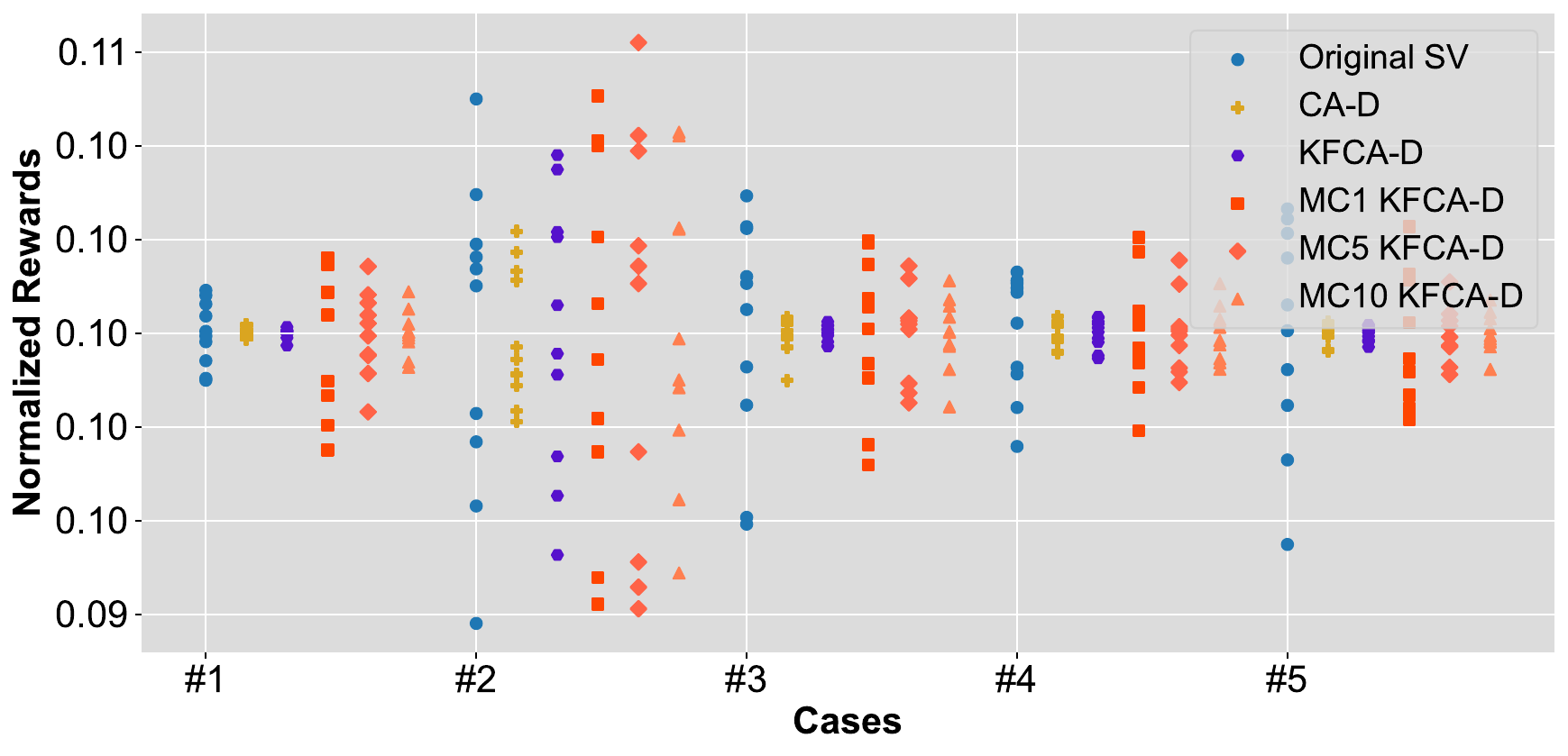}\\[0.3em]
        \includegraphics[width=\linewidth]{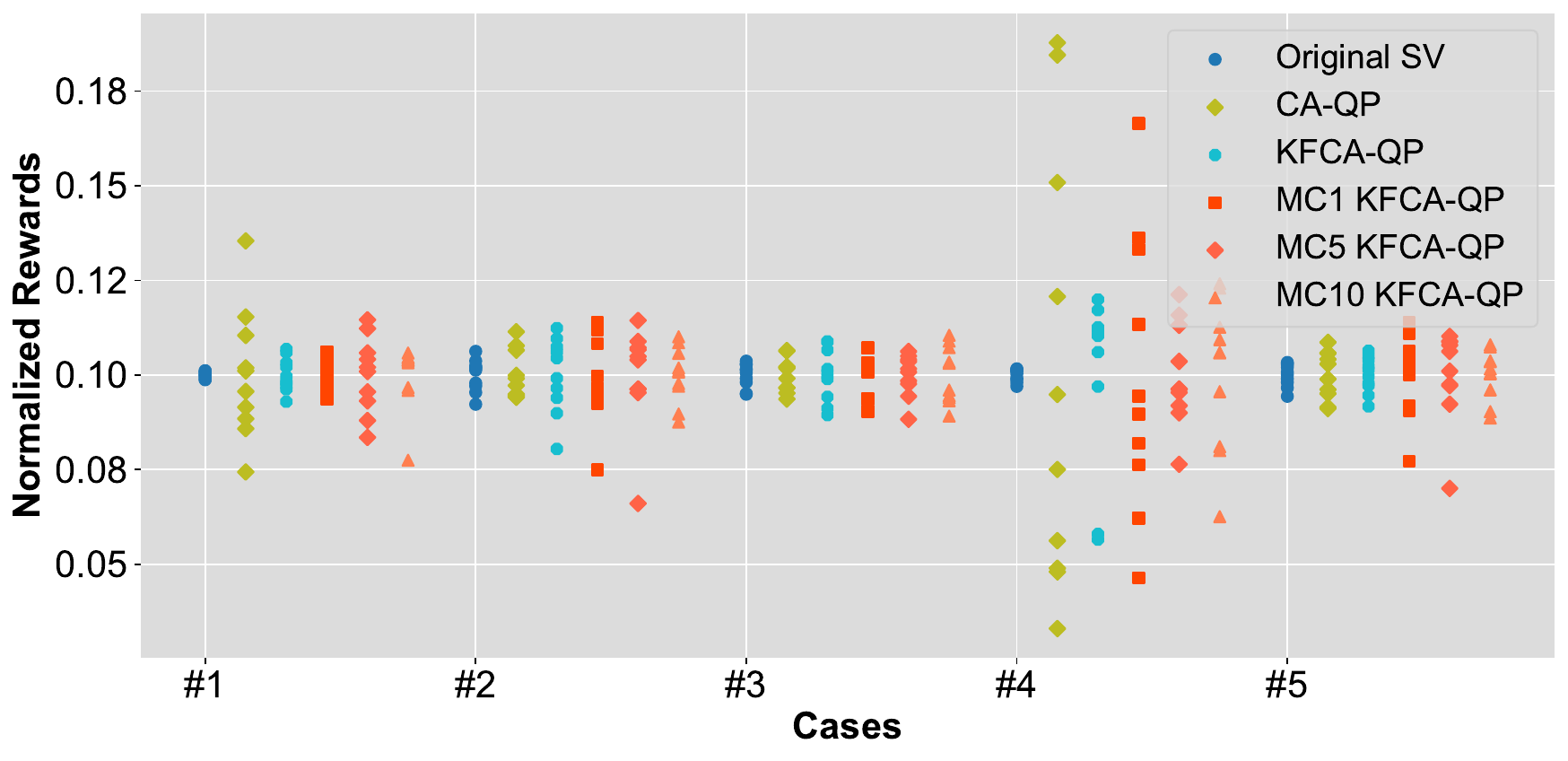}
        \caption{}\label{fig:sv_flca_comparison}
    \end{subfigure}
    \begin{subfigure}[c]{0.5\textwidth}
        \centering
        \includegraphics[width=0.78\linewidth]{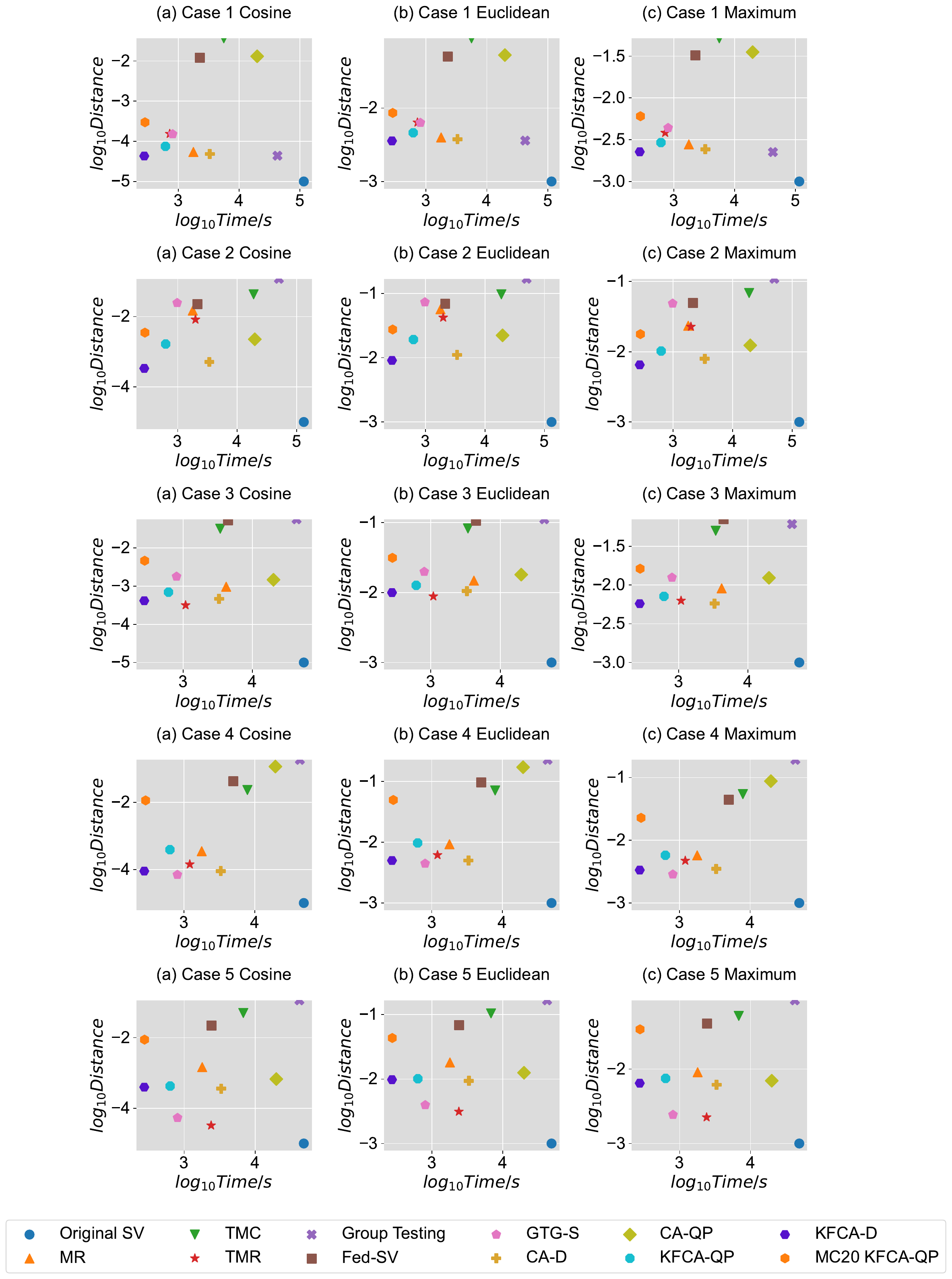}
        \caption{}\label{fig:all_metrics}
    \end{subfigure}
    \caption{(a)~Variance of reward distributions for KFCA, CA, and exact SV (top: public-dataset methods; bottom: quantized-parameter methods). (b)~Reward distance to exact SV plotted against computation time across five cases.}
\end{figure}
In Figure~\ref{fig:sv_flca_comparison} we compare the variance of our KFCA methods with CA-D \cite{LIU_CA} and CA-QP \cite{CA_HONGTAO} with respect to the exact Shapley value. Both KFCA-D and KFCA-QP better resemble the exact Shapley value distribution across all five cases. Details about the individual reward distributions are in Appendix, Figure~\ref{app:rewarddistributionsingleclient}.

Figure~\ref{fig:all_metrics} maps the computation time with its distance to the exact Shapley value for different CA mechanisms and state-of-the-art Shapley value estimation methods. Details about all methods are in Appendix~\ref{app:comparisonmethodsSV}. We measured three distance metrics, namely, (i) cosine distance ($1 - \cos(\phi^{*}, q)$), (ii) Euclidean distance ($\sqrt{\sum_{i=1}^{N} (\phi^{*}_{i} - q_{i})^{2}}$), and (iii) the maximum difference ($\max_{i=1}^{N}|\phi^{*}_{i} - q_{i}|$), where $\phi^{*}$ is the exact Shapley value and $q$ the normalized reward of the respective method. For Shapley value estimation methods, we applied the stopping-threshold $\frac{1}{N * l} \sum_{l=1}^{10} \sum_{i=1}^N \frac{\left|\phi_i^h-\phi_i^{h-l}\right|}{\left|\phi_i^h\right|}<0.05$ from \cite{gtgShapley} where $h$ is the evaluation round number, $l$ is the length of permutation positions that are calculated without random sampling, and $N$ is the total number of clients.
We normalize the reward of each client according to $q_i = \frac{q_i}{\sum_{i=1}^{N} q_i }$. The results show that KFCA is orders of magnitude more efficient than its CA counterpart while resembling the exact Shapley value better on average.



\subsection{KFCA in Quality Inspection: A Real-World Application}
\label{sec:pcb}

\begin{figure}[!htbp]
    \centering
    \begin{subfigure}[c]{0.42\textwidth}
        \centering
        \includegraphics[width=\linewidth]{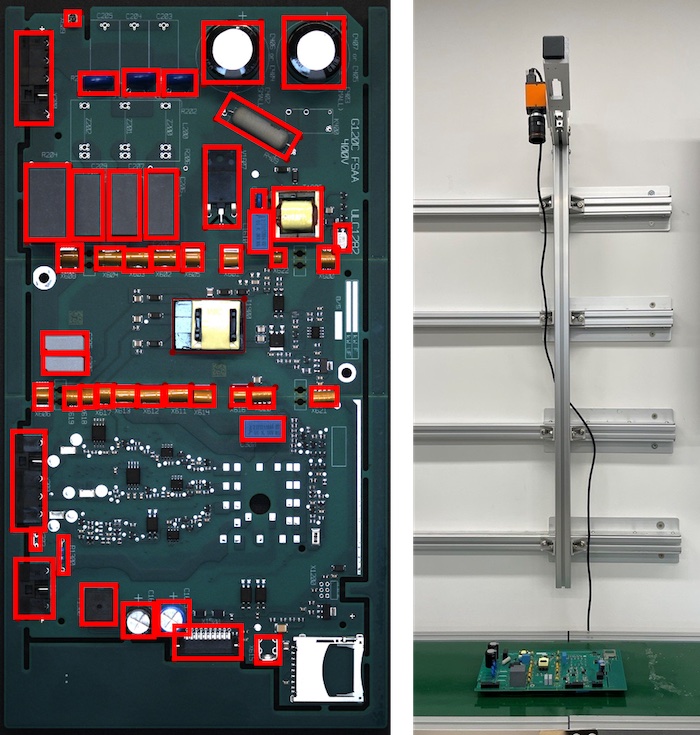}
        \caption{}\label{fig:siemens_pcb}
    \end{subfigure}
    \hfill
    \begin{subfigure}[c]{0.54\textwidth}
        \centering
        \begin{subfigure}[t]{\linewidth}
            \centering
            \includegraphics[width=\linewidth]{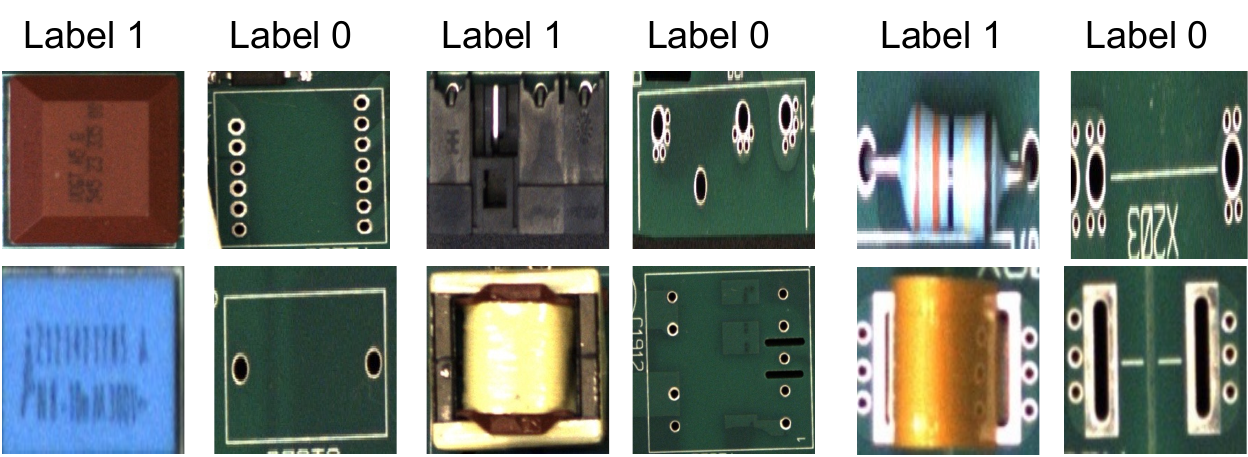}
            \caption{}\label{fig:pcb_binary}
        \end{subfigure}\\[0.3em]
        \begin{subfigure}[t]{\linewidth}
            \centering
            \includegraphics[width=0.49\linewidth]{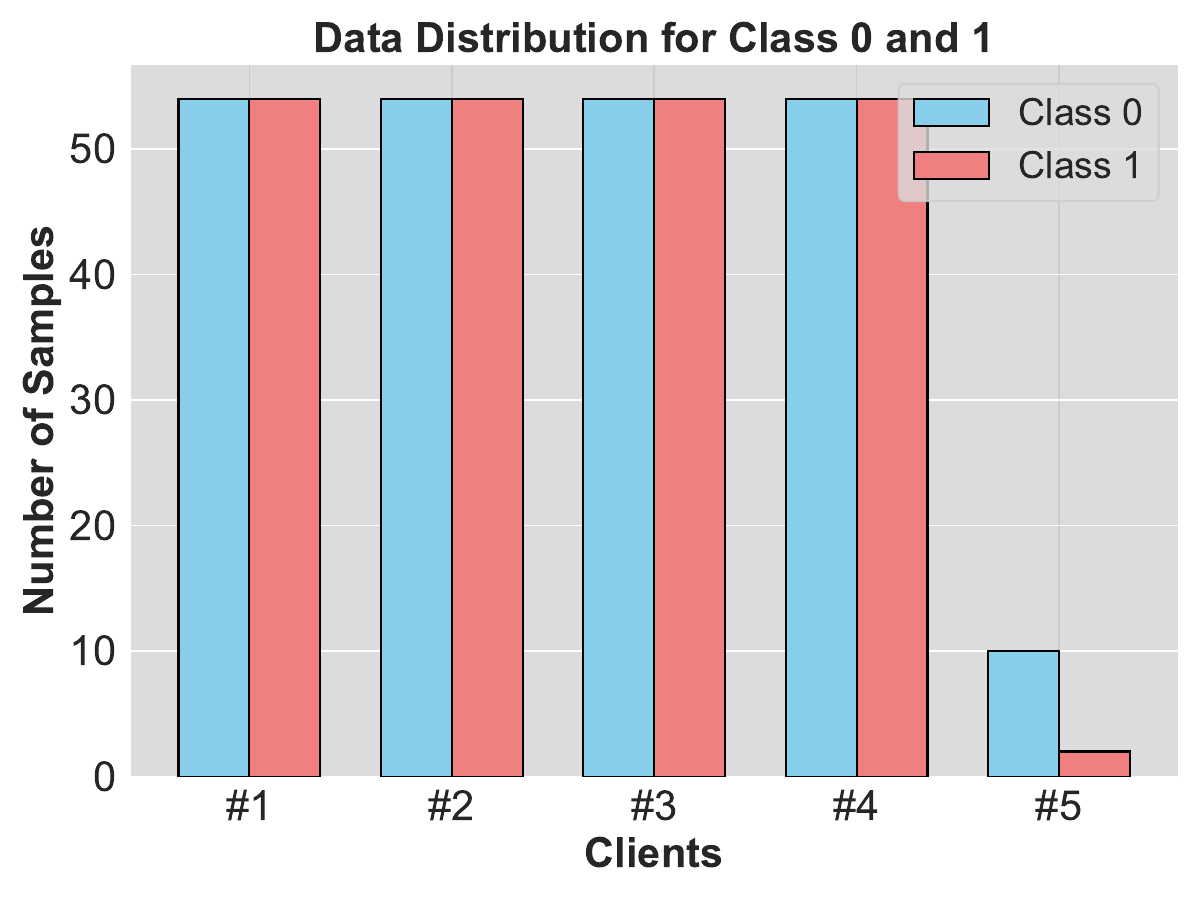}\hfill
            \includegraphics[width=0.49\linewidth]{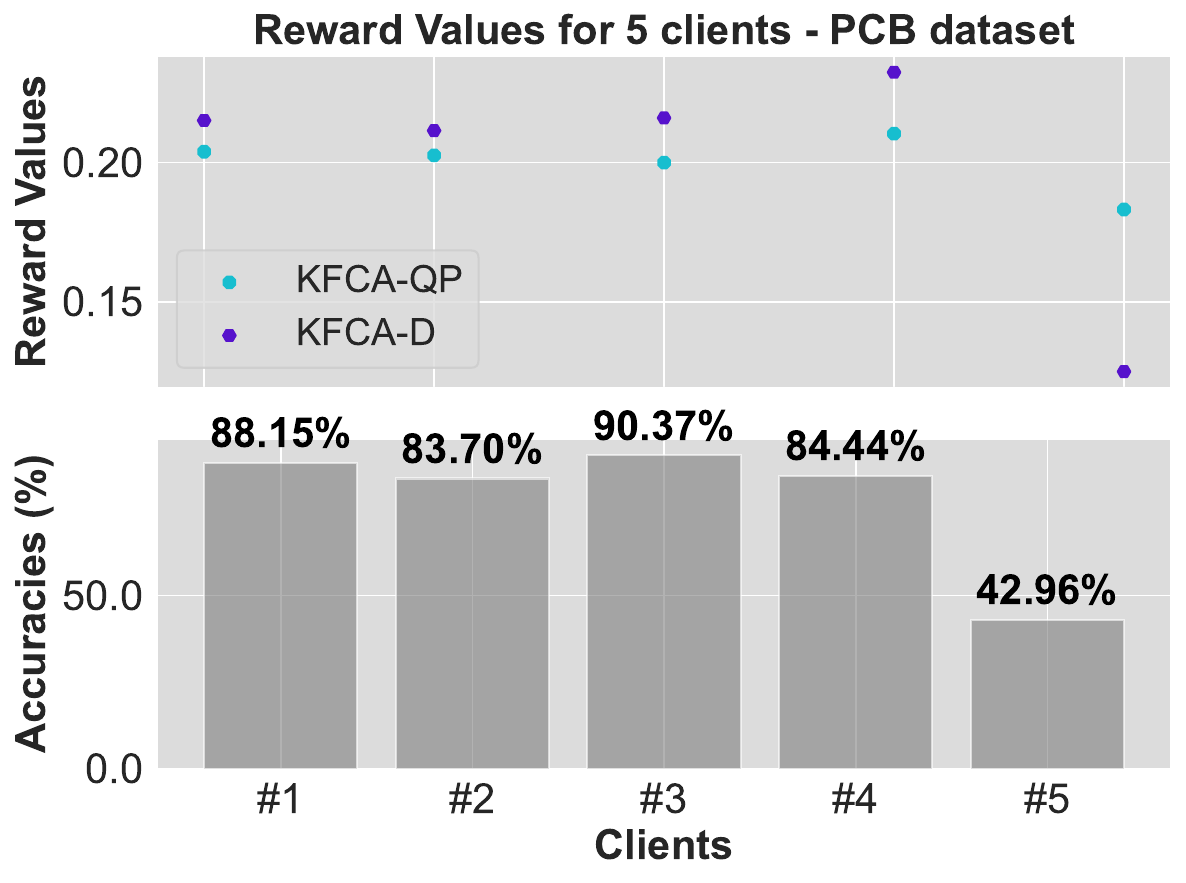}
            \caption{}\label{fig:siemens_comparison}
        \end{subfigure}
    \end{subfigure}
    \caption{Real-world PCB inspection deployment with KFCA-QP. (a)~Hardware: PCB board with manually installed components and camera stand. (b)~Representative PCB component samples. (c)~Per-client data distribution (left) and KFCA-QP rewards (right) across five FL clients.}
\end{figure}
We deployed KFCA-QP in a real-world production PCB assembly inspection pipeline, where line-mounted cameras capture component images for automated defect detection (Fig.~\ref{fig:siemens_pcb}). This setting is a natural fit for incentivized FL: data is siloed across manufacturing sites and cannot be pooled, and maintaining a shared public test set is often impractical, making Shapley-style rewards unrealistic. We evaluate the deployed regime via federated binary classification (component correctly assembled vs.\ missing) by training LeNet~\cite{lecun1989backpropagation} with FedAvg across $5$ clients on $672$ images; representative samples are shown in Fig.~\ref{fig:pcb_binary}. KFCA-QP computes rewards from communicated (quantized) updates alone and down-weights a low-quality client while assigning similar rewards to the remaining sites (Fig.~\ref{fig:siemens_comparison}).

\subsection{KFCA for Federated LLM Fine-Tuning}
\label{sec:flowertune_main}

Federated LLM fine-tuning is an increasingly important FL setting because high-value domain data (e.g., finance, healthcare, and enterprise code) is often inaccessible for centralized training due to privacy, security, or governance constraints. We instantiate KFCA-QP on the top FlowerTune LLM Leaderboard configurations (January 2026)~\cite{flowertune_leaderboard_website} across four domains (general NLP, finance, medical, code), aligning our evaluation with realistic, state-of-the-art federated adapter-tuning scenarios (Figure~\ref{fig:LLM_KFCAQP})~\cite{gao2025flowertune}. Following FlowerTune, clients fine-tune only lightweight LoRA/DoRA adapters and communicate only these adapter updates, reflecting practical deployments where a shared public evaluation set may be unavailable.


\begin{figure}[!htbp]
    \centering
    \begin{subfigure}[c]{0.55\textwidth}
        \centering
        \includegraphics[width=\linewidth]{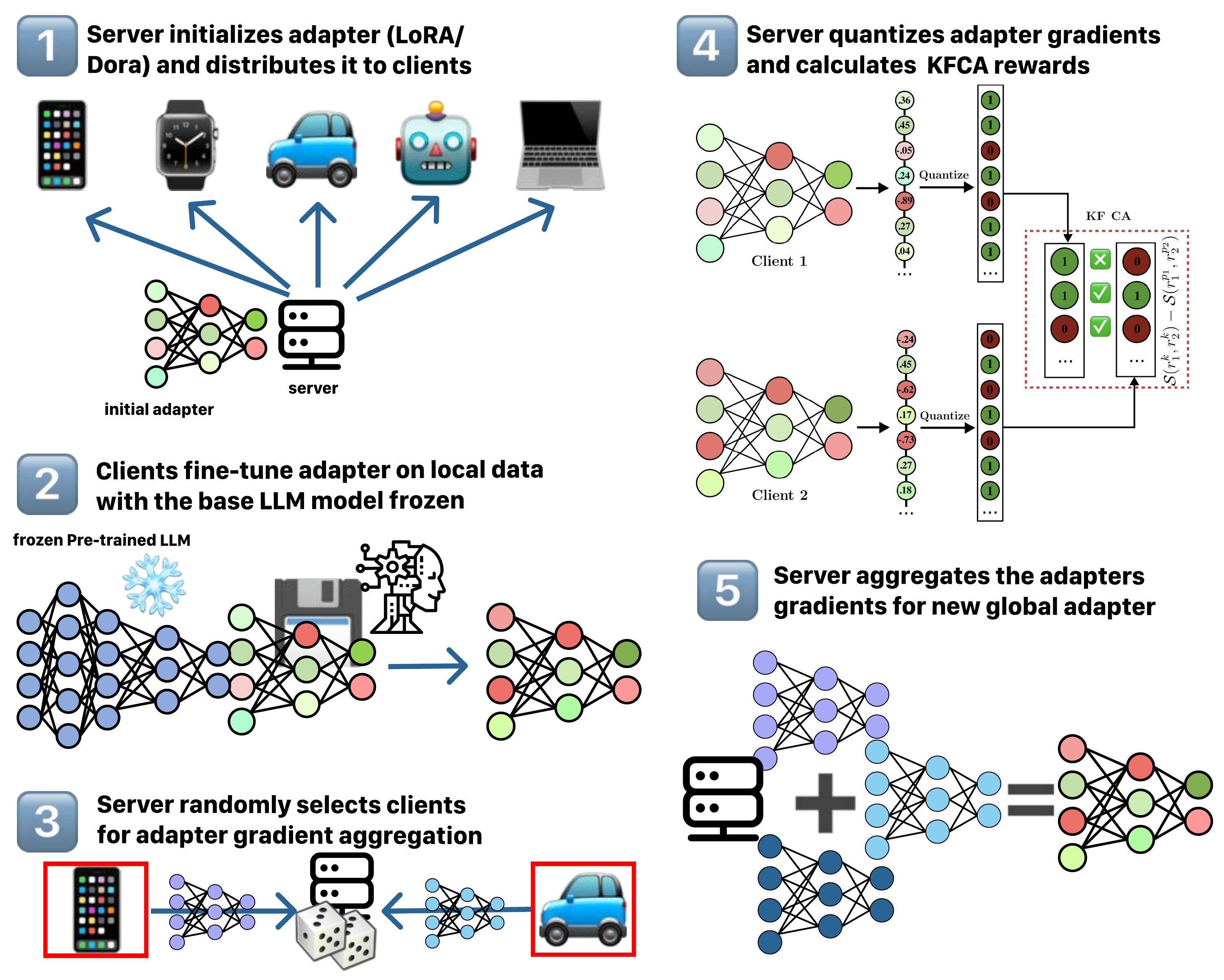}
        \caption{Per-round pipeline.}\label{fig:LLM_KFCAQP}
    \end{subfigure}
    \hfill
    \begin{subfigure}[c]{0.42\textwidth}
        \centering
        \includegraphics[width=\linewidth]{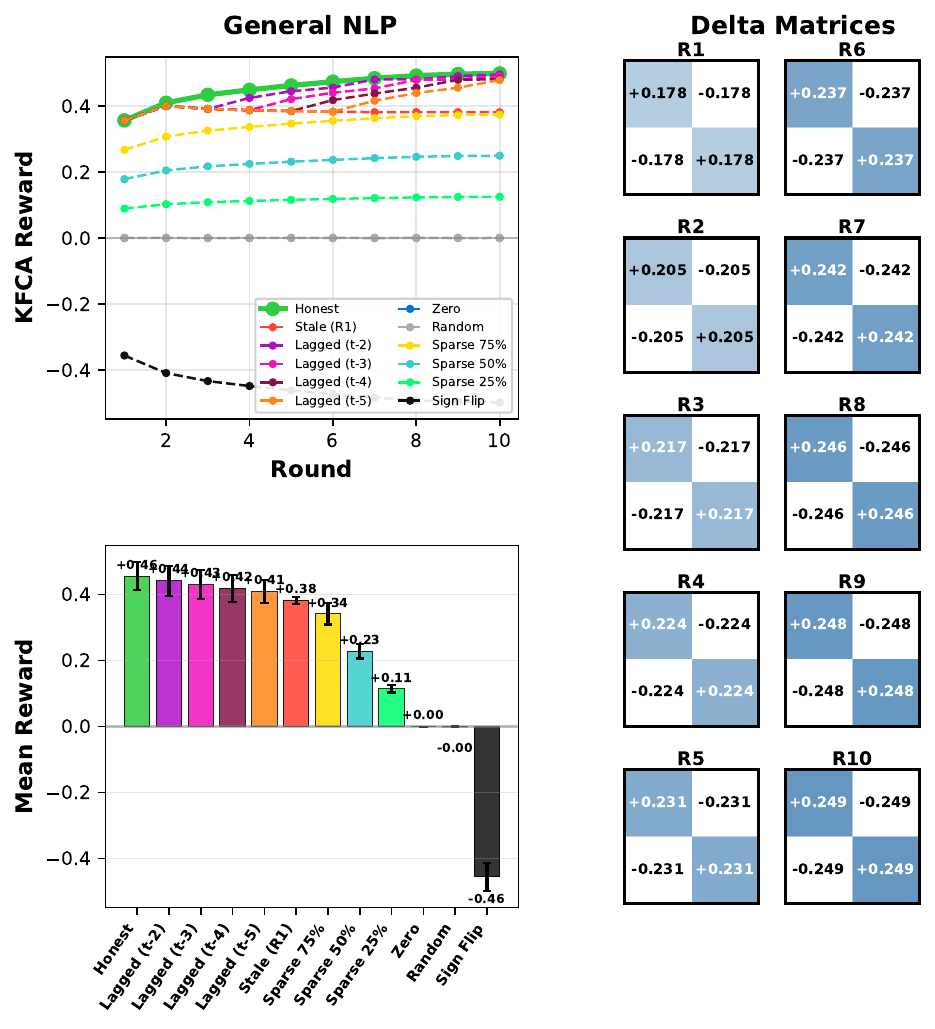}
        \caption{Reward separation (general NLP).}\label{fig:Honest_Reward_Delta_Teaser}
    \end{subfigure}
    \caption{Federated LLM adapter tuning with KFCA-QP. (a)~Per-round pipeline: server distributes LoRA/DoRA adapters, clients fine-tune locally with frozen base LLM, adapter updates are sign-quantized and rewarded by peer agreement, then aggregated. (b)~Reward separation between honest and attack strategies on general NLP (full results in Appendix~\ref{app:llm_finetune}).}
\end{figure}

\textbf{Results:} Across all four domains and all 10 rounds, 1-bit–quantized LoRA/DoRA adapter updates satisfy the categorical-world condition. Under this structure, KFCA-QP yields a strictly incentive-aligned reward ordering: honest updates receive the highest rewards, while deviations are penalized monotonically—free-riding collapses to $\approx 0$ reward and adversarial sign-flips become strongly negative, with intermediate manipulations graded by their departure from honesty. An example for the general NLP category is shown in Figure~\ref{fig:Honest_Reward_Delta_Teaser}. The full experimental results are in Appendix~\ref{app:llm_finetune}.

\FloatBarrier
\section{Conclusion and Future Research}


We introduced Knowledge-Free Correlated Agreement (KFCA), a multi-task peer-prediction reward mechanism for federated learning that does not require knowledge about report distributions or the ground truth. Under a categorical-world condition, KFCA is strongly truthful and eliminates the label-flipping vulnerability of Correlated Agreement (CA), while enabling lightweight, real-time reward computation. Beyond controlled Shapley-value benchmarks, we demonstrate KFCA's practical applicability in state-of-the-art federated LLM adapter fine-tuning and in a real-world PCB quality inspection setting. In future work, we aim to explore how KFCA can be implemented within smart contracts to coordinate and incentivize swarms of autonomous AI agents; we have outlined a decentralized, blockchain-based workflow in Appendix~\ref{app:blockchain}.

\bibliography{references}
\bibliographystyle{plainnat}

\newpage
\appendix

\section{Appendix}
\subsection{Background}

\begin{table*}[h]
    \centering
    \caption{Comparison of Rewarding and Contribution Measuring Methods for Federated Learning.}
    \label{tab:comparison}
    \begin{tabular}{p{.1\textwidth}p{.3\textwidth}p{.3\textwidth}p{.2\textwidth}}
        \toprule
        \textbf{Type} & \textbf{Description} & \textbf{Limitations} & \textbf{Examples} \\
        \midrule
        Honest Report/Full Information & Rewards based on client's reported data, bandwidth, accuracy, loss, cost. & Prone to dishonest behavior; may require Trusted Execution Environments on unsuitable devices. & \cite{LW6_data1, LW11_RewardResponseGame, GTG_28} \\
        \midrule
        Reputation Systems & Rewards based on client's reputation or majority voting. & Quantifying reputation is an open research topic. & \cite{LW5_FL_HomeAppliances_IoT, LW7_Zhang_Reputation} \\
        \midrule
        Leave-one-out & Measures contributions using a performance metric, e.g., accuracy. & Overlooks interactions between datasets. & \cite{leaveoneout} \\
        \midrule
        Shapley Value & Accounts for all interactions and promotes fairness (appendix (\ref{app:shapley_value}). & May be computationally intensive in its original form. & \cite{shapleyMLicml, AISTATSshapleyML, gtgShapley} \\
        \midrule
        Similarities & Rewards based on model similarities or subjective ratings. & Prone to attacks; lacks theoretical and experimental rigor. & \cite{crowdsensing_meets_FL, LW10_refiner, generatedTestset_majorityvote} \\
        \midrule
        Peer Prediction & Uses methods to elicit truthful behavior without direct ground truth. & Requires knowledge of client report distributions. & \cite{LW1_witt2021rewardbased, LIU_CA, CA_HONGTAO} \\
        \bottomrule
    \end{tabular}
\end{table*}

\label{app:background}
In FL, measuring contributions can be challenging due to its inherent privacy-preserving nature. Information about data, such as its quality or distribution, and the quality of training rounds remain private to the respective clients and cannot be directly verified, as only the model updates are visible to the central server. This lack of explicit information poses challenges in determining the influence of a reported model update on the overall quality of FL training and is open research in the FL community \cite{LeonSP, AdvancesAndOpenProblemsFL, LW1_witt2021rewardbased}. The contributions of clients in FL can be (i) self-reported \& reputation, (ii) measured explicitly as a performance metric on a testset or (iii) implicitly by comparing reports \cite{LeonSP}. This section summarizes the existing contribution measurement.

\subsubsection{Honest Report/Full Information}
\textbf{Honest Report/Full Information} rewards are calculated on the basis of the client's report of the amount of data \cite{LW6_data1, LW11_RewardResponseGame, GTG_28, GTG_31, GTG_5, LW4_FL_IIoT}, bandwidth \cite{GTG_31, GTG_4}, local accuracy \cite{MH1_Chai2021, 10_GTG, LW4_FL_IIoT} or local loss \cite{EntropyLossonFullInfo} or cost (data collection, training, etc) \cite{CrowdsourcingFrameworkFL, 10_GTG, GTG_4, GTG_22}. Yet reward systems based on such simplified assumptions may not be applicable in any real-world scenario as the dominant strategy for an individual-rational agent is dishonest behavior (report the best possible outcome without costly model training). Applying Trusted Execution Environments \cite{TXX} might solve the issue, yet it might be an infeasible technical requirement for mobile, edge, or IoT devices.

\textbf{Reputation systems} Reward mechanisms based on the client's reputation \cite{LW5_FL_HomeAppliances_IoT, LW7_Zhang_Reputation, 10_GTG, MH13_Li2020} or majority voting \cite{crowdsensing_meets_FL, CrowdsourcingFrameworkFL, KT02_Toyoda2020Access, KT01_Toyoda2019BigData} promises to relax heavy verification and control mechanics for high reputation clients. How to quantify the reputation in a fair and robust fashion remains open research.

\subsubsection{Direct Comparisons on Performance Metrics}
\textbf{Leave-one-out} Assessing client contributions in a federated learning (FL) setting can be approached by defining a performance metric (e.g., accuracy) and subsequently evaluating these contributions, such as updated models, on a public data- or testset $D^{pub}$. A prevalent method is the Leave-one-out strategy \cite{leaveoneout}, where a client's contribution is determined by the performance differential of the global model with and without its participation. However, this seemingly straightforward methodology overlooks contributions or datasets' interactions between clients.

\textbf{Shapley value} In contrast, the Shapley value, an idea rooted in cooperative game theory, accounts for all conceivable interactions, promoting fairness through principles of efficiency, symmetry, linearity, and null-player \cite{shapley1953value}, by assigning each client a unique value according to their marginal contribution. Even though imposing exponential computation in its original form, optimizations and approximations make SV a suitable method to assess contributions in the FL context \cite{shapley2,shapleyMLicml,ShapleyFLSpringer,AISTATSshapleyML,gtgShapley, ProfitAllocationFL, wang2020principled, wei2020efficient, LW9_FedCoin,MH4_Ma2021}.

\subsubsection{Relative Comparisons}
\textbf{Similarities}
In FL, some works base the reward distribution on proxy metrics like the distance between the model parameters and the average model \cite{crowdsensing_meets_FL, LW5_FL_HomeAppliances_IoT}, subjective rating of others on locally owned data \cite{KT01_Toyoda2019BigData,KT02_Toyoda2020Access, LW10_refiner} or on
majority votes on pairwise simrlarities on artificially generated data \cite{generatedTestset_majorityvote}, local relative accuracy levels \cite{LW11_RewardResponseGame} etc. However, these methods lack theoretical properties and thorough experimental analysis and assume that local models which are similar to the aggregated model are of more value, incentivizing attacks to manipulate the overall outcome for their own reward.

\textbf{Peer prediction}
In FL, valuable contributions are not explicitly visible. However, Multi-task Peer prediction methods can elicit informed and truthful behavior from clients without knowing the ground truth, e.g. what is a ``high-value contribution'' by leveraging correlations between multiple reports. Ideally, a scoring rule is defined that allocates rewards such that clients are incentivized to (i) train the model and (ii) report it truthfully when the ground truth is not directly observable. The Peer Truth Serum \cite{LW1_witt2021rewardbased} and Correlated Agreement (CA) \cite{LIU_CA, CA_HONGTAO} have been applied to incentivize FL. However, designing the scoring rule with said behavior requires knowledge about the report distributions of all clients to be available.

\subsection{Limitations of Correlated Agreement: Details}
\label{app:ca_limitations}

This appendix expands on the two CA limitations summarised in Section~\ref{sec:ca_limitations}: (i)~the practical obstacles introduced by CA's reliance on the global $\Delta$ matrix, and (ii)~a worked binary example showing that label flipping is rewarded identically to truthful reporting under CA.

\subsubsection{Knowledge on the report distribution}
\label{app:ca_knowledge}

CA requires computing the correlation matrix $\Delta$ defined in Definition~\ref{def:delta}, which depends on both the joint distribution $\mathbb{P}(Z_1,Z_2)$ and the marginal distributions $\mathbb{P}(Z_1)$ and $\mathbb{P}(Z_2)$. This introduces three practical obstacles:
\begin{enumerate}
    \item \emph{Centralization requirement.} Estimating $\Delta$ requires access to all client reports, violating the decentralized and privacy-preserving nature of FL.
    \item \emph{Computational inefficiency.} Per pair, the server scans $m$ joint reports to populate an $L{\times}L$ count table and derive $\hat{\Delta}$ ($\mathcal{O}(m+L^2)$); across $\binom{n}{2}$ pairs, $\mathcal{O}(n^2(m+L^2))$ per round, dominated by $\mathcal{O}(n^2 m)$ in FL where $m \gg L^2$.
    \item \emph{Delayed reward computation.} Because $\Delta$ must be globally estimated, CA cannot compute payments in real time, limiting deployability in online or blockchain-based implementations.
\end{enumerate}

\subsubsection{Worked label-flipping example}
\label{app:ca_labelflip_example}

Substituting the CA score matrix from Definition~\ref{def:CA} into Eq.~\eqref{eq:expected_reward_delta}, the expected reward under deterministic reporting functions $f_1, f_2 : [L] \to [L]$ simplifies to:
\begin{equation}
\label{eq:expected_reward_CA}
E(f_1, f_2)
= \sum_{a,b \in [L]} \Delta(a,b)\, \mathcal{S}_{\mathrm{CA}}\big(f_1(a), f_2(b)\big).
\end{equation}
Consider a binary task ($L=2$) where client~1 reports truthfully across six tasks: $(1,0,1,0,1,0)$, and client~2, though informed, flips the labels: $(0,1,0,1,0,1)$. The empirical distributions are
\[
\begin{aligned}
&\mathbb{P}(Z_1 = 0) = \mathbb{P}(Z_1 = 1) = 0.5, \\
&\mathbb{P}(Z_2 = 0) = \mathbb{P}(Z_2 = 1) = 0.5, \\
&\mathbb{P}(Z_1 = 0, Z_2 = 1) = \mathbb{P}(Z_1 = 1, Z_2 = 0) = 0.5,
\end{aligned}
\]
with all other joint probabilities zero. The delta matrix is
\[
\Delta =
\begin{pmatrix}
-0.25 & 0.25 \\
0.25 & -0.25
\end{pmatrix}.
\]
Using the CA scoring rule $\mathcal{S}_{\mathrm{CA}}$ from Definition~\ref{def:CA}, the expected reward becomes
\[
E(f_1, f_2) = (-0.25)\cdot 0 + 0.25\cdot 1 + 0.25\cdot 1 + (-0.25)\cdot 0 = 0.5,
\]
which matches the payoff under truthful reporting. A client who systematically inverts labels therefore receives the maximal reward, exposing CA's susceptibility to label-flipping attacks.

\subsection{Formal Assumptions for Multi-Task Peer Prediction}
\label{app:MTTPassumptions}

We adopt the standard assumptions from the multi-task peer prediction (MTPP) literature~\cite{DasguptaGhosh2013, OriginalCA}. Let \( N \) denote the set of clients and \( M \) the set of tasks.

\begin{enumerate}
    \item \textbf{Payment feasibility.}  
    The mechanism can assign (e.g., monetary) rewards to clients based on their reports.
       \item \textbf{Risk-neutrality.}  
    Clients are risk-neutral and only care about their own payment and effort cost. The utility for client \( i \) is:
    \[
    U_i = E_i - c(e_i), \quad \text{where } c(1) > c(0) = 0,
    \]
    with \( E_i \) the expected total payment across tasks, and \( c(e_i) \) the total effort cost.
    \item \textbf{Binary effort.}  
    Each client \( i \in N \) independently decides whether to exert effort on each task \( k \in M \), denoted:
    \[
    e_i^k \in \{0,1\}, \quad \text{where } 1 = \text{effortful},\; 0 = \text{shirking}.
    \]
    
    \item \textbf{Ex-ante identical tasks.}  
    All tasks are drawn independently from a common prior distribution over ground-truth labels.  
    Clients cannot distinguish tasks before seeing their signals. 
\begin{itemize}
    \item \textbf{Signal generation.}  
    Each task \( k \in M \) has an unknown ground-truth label \( Y^k \in [L] \).  
    Each client \( i \) receives a private signal \( Z_i^k \in [L] \), sampled based on their effort level:
    \[
    \mathbb{P}(Z_i^k = a \mid Y^k = y, e_i^k) =
    \begin{cases}
        P_i(a \mid y), & \text{if } e_i^k = 1, \\
        Q_i(a),        & \text{if } e_i^k = 0,
    \end{cases}
    \]
    where:
    \begin{itemize}
        \item \( P_i(\cdot \mid y) \) is an informative channel, assumed diagonally dominant:  
        \[
        P_i(y \mid y) > P_i(a \mid y), \quad \forall a \ne y;
        \]
        \item \( Q_i(\cdot) \) is an uninformative distribution, independent of \( y \).
    \end{itemize}
    When all clients exert effort (\( e_i^k = 1 \)), signals are conditionally independent given \( Y^k \):
    \[
    \mathbb{P}(Z_1^k = a, Z_2^k = b \mid Y^k = y) = P_1(a \mid y)\, P_2(b \mid y).
    \]

    \item \textbf{Report generation.}  
    After receiving signal \( Z_i^k \), each client reports a label \( R_i^k \in [L] \) using a (possibly randomized) strategy \( F_i(r \mid z) \).  
    The conditional probability of reporting \( r \) is:
    \[
    \mathbb{P}(R_i^k = r \mid Y^k = y, e_i^k)
    = \sum_{a \in [L]} F_i(r \mid a)\, \mathbb{P}(Z_i^k = a \mid Y^k = y, e_i^k).
    \]
\end{itemize}

\end{enumerate}

This framework models strategic agents who decide whether to exert effort and may misreport their signals.
These assumptions are minimal yet expressive, ensuring the signal structure needed for peer-prediction to detect effort and incentivize truthful reporting.
In the main paper, we instantiate this framework with the CA and KFCA mechanisms.

\paragraph{Informed vs.\ uninformed strategies.}
Reporting strategies $F_i$ are classified by whether the report distribution depends on the observed signal:
\begin{itemize}
    \item \emph{Informed strategy:} there exist $a \ne b \in [L]$ such that $F_i(\cdot \mid a) \ne F_i(\cdot \mid b)$ — the report uses information from the signal.
    \item \emph{Uninformed strategy:} $F_i(\cdot \mid a) = F_i(\cdot \mid a')$ for all $a, a' \in [L]$ — the report distribution does not depend on the signal.
\end{itemize}
The truthful strategy $f^\star(a) = a$ is the canonical informed strategy; constant-relabeling and label permutations are the principal informed deviations of interest.
\subsection{Expected Reward}
\label{app:expectedvalue}

Let \( F_1, F_2 : [L] \to \Delta([L]) \) be the reporting strategies of two clients, where \( F_1(r_1 \mid a) \) and \( F_2(r_2 \mid b) \) denote the probability of reporting \( r_1 \) and \( r_2 \), given private signals \( Z_1 = a \) and \( Z_2 = b \), respectively. Then the expected reward under scoring rule \( \mathcal{S} : [L] \times [L] \to \mathbb{R} \) is:

\begin{align}
\begin{split}
\label{eq:expected_reward}
E(F_1, F_2) &= \sum_{a=1}^L \sum_{b=1}^L \mathbb{P}(Z_1 = a, Z_2 = b) \\
&\quad \times \sum_{r_1=1}^L \sum_{r_2=1}^L F_1(r_1 \mid a)\, F_2(r_2 \mid b)\, \mathcal{S}(r_1, r_2) \\
&\quad - \sum_{a=1}^L \sum_{b=1}^L \mathbb{P}(Z_1 = a)\, \mathbb{P}(Z_2 = b) \\
&\quad \times \sum_{r_1=1}^L \sum_{r_2=1}^L F_1(r_1 \mid a)\, F_2(r_2 \mid b)\, \mathcal{S}(r_1, r_2).
\end{split}
\end{align}

Using the delta matrix \(\Delta(a,b) := \mathbb{P}(Z_1 = a, Z_2 = b) - \mathbb{P}(Z_1 = a)\mathbb{P}(Z_2 = b)\), this simplifies to:

\begin{align}
\begin{split}
\label{eq:expected_reward_simplified_appendix}
E(F_1, F_2) &= \sum_{a=1}^L \sum_{b=1}^L \Delta(a, b) \sum_{r_1=1}^L \sum_{r_2=1}^L \mathcal{S}(r_1, r_2)\, F_1(r_1 \mid a)\, F_2(r_2 \mid b).
\end{split}
\end{align}

In the case of deterministic reporting strategies \( f_1, f_2 : [L] \to [L] \), where each client reports a fixed label given its signal, i.e.,
\[
F_1(r \mid a) = \mathds{1}\{r = f_1(a)\}, \quad F_2(r \mid b) = \mathds{1}\{r = f_2(b)\},
\]
the expected reward simplifies to:
\begin{equation}
E(f_1, f_2) = \sum_{a=1}^L \sum_{b=1}^L \Delta(a, b)\, \mathcal{S}(f_1(a), f_2(b)).
\end{equation}
\subsection{Proof of Theorem~\ref{thm:CA_truthful}: Informed Truthfulness of CA}
\label{app:truthfulproof}

We reproduce the informed-truthfulness result of the Correlated Agreement (CA) mechanism following~\cite{OriginalCA}. The goal is to show that truthful reporting yields weakly higher expected reward than any other deterministic strategy profile, with equality only for informed strategies.

\paragraph{Setup and Assumptions.}
We assume the following standard conditions:
\begin{enumerate}
    \item Tasks are ex-ante identical and randomly assigned; clients cannot distinguish between bonus and penalty tasks.
    \item Clients receive conditionally independent signals \( Z_1, Z_2 \in [L] \), given the latent label \( Y \in [L] \).
    \item Clients apply deterministic reporting functions \( f_1, f_2 : [L] \to [L] \).
\end{enumerate}

Define the \emph{delta matrix} \( \Delta \in \mathbb{R}^{L \times L} \) as:
\[
\Delta(a, b) := \mathbb{P}(Z_1 = a, Z_2 = b) - \mathbb{P}(Z_1 = a)\, \mathbb{P}(Z_2 = b),
\]
which is symmetric and mean-centered:
\[
\Delta(a,b) = \Delta(b,a), \quad \sum_b \Delta(a,b) = \sum_a \Delta(a,b) = 0.
\]

The CA mechanism defines the scoring rule:
\[
\mathcal{S}_{\mathrm{CA}}(a,b) :=
\begin{cases}
1, & \text{if } \Delta(a,b) > 0, \\
0, & \text{otherwise}.
\end{cases}
\]

\begin{proof}
Let \( f_1, f_2 : [L] \to [L] \) be arbitrary deterministic reporting functions. 

The expected reward under the CA mechanism is:
\[
E(f_1, f_2) = \sum_{a,b \in [L]} \Delta(a,b) \cdot \mathcal{S}_{\mathrm{CA}}(f_1(a), f_2(b)).
\]

\begin{enumerate}
    \item \textbf{Truthful strategy:} For the truthful profile \( f^\star(a) = a \), we have:
    \[
    E(f^\star, f^\star) = \sum_{a,b \in [L]} \Delta(a,b) \cdot \mathds{1} \{ \Delta(a,b) > 0 \} = \sum_{\Delta(a,b) > 0} \Delta(a,b).
    \]

    \item \textbf{Arbitrary strategy:} For any other strategy pair \( (f_1, f_2) \), the expected reward is:
    \[
    E(f_1, f_2) = \sum_{a,b \in [L]} \Delta(a,b) \cdot \mathds{1} \{ \Delta(f_1(a), f_2(b)) > 0 \} \le \sum_{\Delta(a,b) > 0} \Delta(a,b),
    \]
    with equality only if \( \Delta(f_1(a), f_2(b)) > 0 \) for all \( (a,b) \) such that \( \Delta(a,b) > 0 \).

    \item \textbf{Uninformed strategies get zero:} Suppose \( f_1 \) is uninformed, e.g., \( f_1(a) = r \) for all \( a \in [L] \). Then:
    \[
    E(f_1, f_2) = \sum_{a,b} \Delta(a,b) \cdot \mathcal{S}_{\mathrm{CA}}(r, f_2(b)) = \sum_b \mathcal{S}_{\mathrm{CA}}(r, f_2(b)) \cdot \sum_a \Delta(a,b) = 0,
    \]
    because \( \sum_a \Delta(a,b) = 0 \) for all \( b \).

\end{enumerate}
Therefore the truthful strategy achieves maximal expected reward. Equality holds only for strategy profiles that preserve all positively correlated signal pairs, i.e., both \( f_1 \) and \( f_2 \) must be informed functions aligned with the latent structure.
\end{proof}
\subsection{Robustness Against Malicious Clients}
\label{app:robustnessproof}

\subsubsection*{Binary Case}

\begin{proposition}[Robustness in Binary Case]
\label{prop:robust-binary}
In the binary label setting (\(L = 2\)), suppose the class prior is uniform (\( \mathbb{P}(Y=0) = \mathbb{P}(Y=1) = 0.5 \)), and honest clients follow symmetric noise with error rate \( \alpha \in [0, 0.5) \), i.e.,
\[
\mathbb{P}(r = Y \mid \text{honest}) = 1 - \alpha, \quad
\mathbb{P}(r \ne Y \mid \text{honest}) = \alpha.
\]
Then the KFCA mechanism preserves truthful incentives for honest clients as long as the fraction \( \lambda \) of malicious clients satisfies \( \lambda < 0.5 \), even under worst-case label-flipping.
\end{proposition}

\begin{proof}
Let \( \lambda \in [0,1] \) denote the malicious client fraction. KFCA uses the MTPP structure with scoring rule \( \mathcal{S}(r_i, r_j) = \mathds{1}\{r_i = r_j\} \), and the expected payment to an honest client is:
\[
E = \mathbb{E}[\text{bonus}] - \mathbb{E}[\text{penalty}].
\]

\textit{Penalty term.} Under class symmetry and symmetric reporting, the marginal distribution of reports is uniform, so:
\[
\mathbb{E}[\text{penalty}] = \sum_{r \in \{0,1\}} \mathbb{P}(r)^2 = \left(\tfrac{1}{2}\right)^2 + \left(\tfrac{1}{2}\right)^2 = \tfrac{1}{2}.
\]

\textit{Bonus term.} There are two cases. If the peer is honest, then since reports are independent given \( Y \), the joint match probability is:
\[
\mathbb{P}(r_i = r_j \mid \text{both honest}) = (1 - \alpha)^2 + \alpha^2 = 1 - 2\alpha(1 - \alpha).
\]
If the peer is malicious and flips the label (i.e., reports \( 1 - r_j^{\text{honest}} \)), the match probability becomes:
\[
\mathbb{P}(r_i = r_j \mid \text{peer flips}) = \mathbb{P}(r_i \ne r_j^{\text{honest}}) = 2\alpha(1 - \alpha).
\]
Thus the expected bonus is:
\[
\mathbb{E}[\text{bonus}] = (1 - \lambda)\left[1 - 2\alpha(1 - \alpha)\right] + \lambda \cdot 2\alpha(1 - \alpha).
\]
Subtracting the penalty:
\[
E = \mathbb{E}[\text{bonus}] - \tfrac{1}{2}
= \left[1 - 2\alpha(1 - \alpha)\right](1 - \lambda) + 2\alpha(1 - \alpha)\lambda - \tfrac{1}{2}.
\]
This simplifies to:
\[
E = (1 - 2\lambda)\cdot \left[(1 - \alpha)^2 + \alpha^2 - \tfrac{1}{2}\right]
= (1 - 2\lambda)\cdot \left[\tfrac{1}{2} - 2\alpha(1 - \alpha)\right].
\]
Since \( \alpha \in [0, 0.5) \), we have \( \tfrac{1}{2} - 2\alpha(1 - \alpha) > 0 \), so \( E > 0 \) if and only if \( \lambda < 0.5 \).
\end{proof}

\subsubsection*{Multi-Class Case}

\begin{proposition}[Robustness in Multi-Class Case]
\label{prop:robust-multiclass}
In the multi-class setting (\( L > 2 \)), the KFCA mechanism remains robust under adversarial reporting, with the tolerable malicious fraction depending on the confusion matrices of honest and malicious clients.
\end{proposition}

\begin{proof}
Let \( \pi_k := \mathbb{P}(Y = k) \) denote the class prior, \( \alpha_{k\ell} := \mathbb{P}(\text{report } \ell \mid Y = k, \text{ honest}) \) the confusion matrix of honest clients, and \( \tilde{\alpha}_{k\ell} := \mathbb{P}(\text{report } \ell \mid Y = k, \text{ malicious}) \) the corresponding matrix for malicious clients. Let \( \lambda \in [0,1] \) be the fraction of malicious clients.

\textit{Bonus term.} When a truthful client is paired with a random peer (honest with probability \( 1 - \lambda \), malicious with probability \( \lambda \)), the expected bonus is:
\[
E_{\text{bonus}} = \sum_{k,\ell} \pi_k\, \alpha_{k\ell} \left[(1 - \lambda)\alpha_{k\ell} + \lambda \tilde{\alpha}_{k\ell}\right]
= (1 - \lambda) A + \lambda B,
\]
where
\[
A := \sum_{k,\ell} \pi_k\, \alpha_{k\ell}^2, \qquad
B := \sum_{k,\ell} \pi_k\, \alpha_{k\ell} \tilde{\alpha}_{k\ell}.
\]

\textit{Penalty term.} Let the marginal report probability be:
\[
q_\ell := (1 - \lambda)\sum_k \pi_k \alpha_{k\ell} + \lambda \sum_k \pi_k \tilde{\alpha}_{k\ell}.
\]
Then the expected penalty is:
\[
E_{\text{penalty}} = \sum_{\ell \in [L]} q_\ell^2.
\]

\textit{Robustness condition.} The total expected reward is:
\[
E_{\text{total}} = E_{\text{bonus}} - E_{\text{penalty}}.
\]
To ensure robustness, we require \( E_{\text{total}} > 0 \). A sufficient condition is:
\[
(1 - \lambda) A + \lambda B > E_{\text{penalty}}.
\]
Solving for \( \lambda \), we obtain:
\[
\lambda < \frac{A - E_{\text{penalty}}}{A - B}, \qquad \text{provided } A > B.
\]

This condition is interpretable: when honest clients exhibit high self-consistency (large \( A \)), malicious clients are less aligned with honest reporting (small \( B \)), and the report marginals are near-uniform (small \( E_{\text{penalty}} \)), then robustness improves.

In particular, under uniform class prior and symmetric confusion (e.g., noisy identity), this condition simplifies and can be estimated empirically.
\end{proof}
\subsection{Why ``up to label permutation'' is a fundamental ceiling, not a KFCA-specific weakness}
\label{app:permutation_indistinguishability}

Proposition~\ref{prop:strong_truthfulness} states that, under the categorical-world condition, KFCA's expected reward is maximised exactly by strategy profiles of the form $f_1 = f_2 = \pi$ for some bijection $\pi : [L] \to [L]$. Since the truthful function $f^\star = \mathrm{id}$ is one such bijection, there are $L!$ reward-maximising profiles, all related by a global relabeling. This subsection clarifies how this fact interacts with KFCA's practical robustness guarantees and explains why the body of the paper describes KFCA as ``strongly truthful under honest majority'' rather than ``strongly truthful up to label permutation.''

\paragraph{Equilibrium uniqueness (population-level fixed points).}
Considered as a set of Bayesian Nash equilibria, KFCA admits one for each global label permutation: if every client agrees to call cats ``dogs'' and dogs ``cats,'' the joint distribution of reports is identical to the honest one and the mechanism cannot tell the two apart. This is a fundamental property of any peer-prediction mechanism without external ground truth---the score function only sees report tuples, so any global relabeling that preserves the joint distribution is reward-equivalent. Both the Correlated Agreement mechanism~\citep{OriginalCA} and the Dasgupta--Ghosh single-task variant~\citep{DasguptaGhosh2013} share this ambiguity, often called \emph{shared-permutation indistinguishability}.

\paragraph{Deviation robustness (out-of-equilibrium analysis).}
What matters for incentive design is not the size of the equilibrium set but whether any individual or coalition can profitably deviate from the prevailing equilibrium. Let $\lambda$ denote the fraction of clients adopting some non-identity permutation $\pi$, with the remaining $(1 - \lambda)$ reporting truthfully. Setting $D := \sum_a \Delta(a, a) > 0$ and $|O| := -\sum_a \Delta(a, \pi(a)) > 0$, the expected reward differential is
\[
E_{\text{honest}} - E_{\pi\text{-flipper}} \;=\; (1 - 2\lambda)\,(D + |O|).
\]
Since $D + |O| > 0$, honest reporting strictly dominates the $\pi$-flipping strategy whenever $\lambda < 1/2$. Permutation attacks therefore offer no advantage over any other minority deviation: they collapse into the same regime covered by Proposition~\ref{prop:robustness}. Above the $50\%$ threshold, $\pi$ becomes the de facto equilibrium and the original honest strategy is the minority deviation---this is the standard Byzantine-fault-tolerance impossibility shared by all peer-prediction mechanisms.

\paragraph{Operational reading.}
In a deployment where fewer than half of the clients collude on any single label permutation---the standard cross-silo FL assumption---KFCA strictly rewards honest reporting over every other deterministic strategy, permutation attacks included. The $L!$ equilibria coexist as theoretical fixed points but are unreachable from honest play under the honest-majority assumption. This is why the body of the paper states KFCA's truthfulness guarantee in terms of a $\lambda < 1/2$ threshold rather than as a caveat about label relabelings.

\subsection{Beyond Ideal Categorical Cases}
\label{app:categoricalpipeline}

In practice, the joint distribution of client signals in federated learning (FL) may deviate from the ideal categorical-world condition (Definition~\ref{def:categorical}). Clients often report continuous gradients, high-dimensional embeddings, or noisy scores derived from heterogeneous local data. In such cases, the empirical correlation matrix \(\Delta\) may not exhibit the diagonal-dominant sign pattern required by KFCA.

In these environments, KFCA cannot be directly applied to raw signals. However, under the following modeling assumptions:
\begin{enumerate}
    \item Signals are conditionally independent given the latent truth;
    \item Each client’s signal is informative about the latent truth,
\end{enumerate}
we can \emph{enforce} a categorical structure through a transformation pipeline that maps arbitrary outputs into categorical representations, producing a correlation matrix
\[
\widetilde{\Delta}(a,b) = \mathbb{P}(\widetilde{Z}_1 = a, \widetilde{Z}_2 = b) - \mathbb{P}(\widetilde{Z}_1 = a)\,\mathbb{P}(\widetilde{Z}_2 = b),
\]
which satisfies the categorical sign condition:
\[
\operatorname{sign}(\widetilde{\Delta}(a,b)) =
\begin{cases}
> 0, & \text{if } a = b, \\
< 0, & \text{if } a \ne b.
\end{cases}
\]
This sign structure ensures that agreement reflects shared latent truth and enables KFCA to remain incentive-compatible and knowledge-free even in complex settings.

We describe a three-stage pipeline that transforms general FL signals into categorical form compatible with KFCA.

\begin{enumerate}
    \item \textbf{Signal-space transformation.}  
    Each client applies a local transformation
    \[
        T_i : \mathcal{Z}_i \to [L'],
    \]
    where \(\mathcal{Z}_i\) is the raw signal space (e.g., vectors, continuous values, or label sets), and \([L']\) is a discrete category space. The transformed report is \(\widetilde{Z}_i = T_i(Z_i)\), designed to retain informativeness about the latent truth while enabling categorical agreement.

    \emph{Common applicable scenarios}:
    \begin{itemize}
        \item \textbf{Continuous or high-dimensional signals:} Raw outputs like gradients, logits, or representations can be discretized or clustered to preserve correlation with underlying labels.
        \item \textbf{Client heterogeneity (non-i.i.d. data):} Clients may have different priors or \(P_i(Z_i \mid Y)\), but if labels are semantically aligned, this heterogeneity is tolerable. When label semantics diverge, alignment is necessary.
    \end{itemize}

    \emph{Typical transformation strategies}:
    \begin{itemize}
        \item \emph{Quantization:} Map vectors to signs, directions, or magnitude bins;
        \item \emph{Clustering:} Use k-means or similar to group semantically similar signals;
        \item \emph{Thresholding:} Binarize performance-related scores (e.g., improvement vs. no improvement).
    \end{itemize}

    \item \textbf{Latent-truth alignment.}  
    Introduce a latent categorical variable \(Y^k \in [L']\) representing the true state for each task. Clients decode this truth via maximum a posteriori (MAP) inference:
    \[
        \widetilde{Z}_i^k = \arg\max_{a \in [L']} \mathbb{P}(Y^k = a \mid Z_i^k) \propto \pi(a)\,P_i(Z_i^k \mid Y^k = a),
    \]
    where \(\pi(a)\) is a prior over latent states. This step ensures cross-client consistency: reports with the same label correspond to the same latent truth, enabling meaningful correlation structure.

    If different clients use inconsistent labels, this stage merges or coarsens categories to create a common latent space.

    \item \textbf{Categorical regularization.}  
    When signals are noisy or sparsely distributed, we apply a smoothing step to the estimated correlation matrix:
    \[
        \widetilde{\Delta}(a,b) := \operatorname{sign}\!\left(\mathbb{E}[\Delta(a,b)]\right) \cdot \left|\mathbb{E}[\Delta(a,b)]\right|^\gamma, \qquad \gamma \in (0,1).
    \]
    This transformation preserves sign structure while enhancing diagonal dominance, ensuring that \(\operatorname{sign}(\widetilde{\Delta}) = \mathbb{I}\).

    \smallskip
    \noindent
    \textbf{Regularization note.}  
    This smoothing step is a modeling device used for stabilization and is not part of the KFCA scoring rule. The matrix \(\mathbb{E}[\Delta(a,b)]\) can be estimated using a small calibration set or prior signal statistics.
\end{enumerate}

\vspace{0.5em}
\noindent
\textbf{Generalized categorical model.}  
After applying the pipeline, each client emits a transformed report \(\widetilde{Z}_i^k \in [L']\) satisfying:
\[
\widetilde{\Delta}(a,a) > 0, \quad \widetilde{\Delta}(a,b) < 0 \text{ for } a \ne b.
\]
Thus, the transformed reports retain all theoretical guarantees of KFCA: strong truthfulness, knowledge-freedom, and computational simplicity.

\begin{proposition}[Categorical Representation Enforcement]
\label{prop:cat_enforce}
Suppose each client signal is transformed into a categorical report \(\widetilde{Z}_i\) via a pipeline that (i) aligns outputs to a common latent truth and (ii) preserves informativeness. Then, under conditional independence of signals given the truth:
\[
    \operatorname{sign}(\widetilde{\Delta}(a,b)) =
    \begin{cases}
        > 0, & \text{if } a = b, \\
        < 0, & \text{if } a \ne b.
    \end{cases}
\]
\end{proposition}

\begin{proof}[Sketch.]
By Lemma~\ref{lem:shirking} and the conditional independence assumption:
\[
    \widetilde{\Delta}(a,b)
    \propto \mathrm{Cov}_Y\big(P_1(\widetilde{Z}_1 = a \mid Y),\,P_2(\widetilde{Z}_2 = b \mid Y)\big).
\]
If \(a = b\), both terms increase when \(Y = a\), giving positive covariance. If \(a \ne b\), no latent state increases both terms, resulting in negative covariance. Alignment ensures a one-to-one mapping between report symbols and latent truths, preserving this sign structure.
\end{proof}

With categorical enforcement in place, KFCA can be applied as:
\[
    s_i^k \xrightarrow{\,T_i\,} \widetilde{Z}_i^k \xrightarrow{\,\text{KFCA}\,} \text{Reward}.
\]
This transformation enables KFCA to handle raw data (e.g. continuous or heterogeneous) by converting signals into categorical reports that satisfy the required correlation structure. As a result, KFCA extends to a broad class of federated learning tasks while preserving its core guarantees: strong truthfulness (up to relabeling), knowledge-freedom, privacy preservation, and applicability to both classification and optimization without relying on external supervision or access to ground truth.
\subsection{The Categorical-World Condition under Non-IID Client Data}
\label{app:non_iid_categorical}

The categorical-world condition (Definition~\ref{def:categorical}) was the single most discussed assumption during the prior review cycle, with reviewers asking whether it remains realistic when clients hold non-IID data. This appendix unpacks the question in detail.

\subsubsection{Non-IID partition vs.\ categorical-world: different objects}

The two concepts live on different layers of the FL pipeline:
\begin{itemize}
    \item \textbf{Non-IID} is a property of the \emph{client data partition}---how training examples are distributed across clients (label skew, quantity skew, feature/domain skew, local label noise).
    \item \textbf{Categorical-world} is a property of the \emph{post-training signal correlation matrix} $\Delta$---specifically, $\Delta(a,a) > 0$ and $\Delta(a,b) < 0$ for $a \ne b$.
\end{itemize}
Non-IID is an upstream input that affects how informative each client's signal is; the categorical-world condition is a downstream sign condition on the joint distribution that informativeness produces.

\subsubsection{When does the condition hold? (binary reduction)}

In the binary case ($L = 2$), the condition reduces to a single inequality on each client's noise rate:
\[
\Delta(a,a) > 0,\ \Delta(a,b) < 0 \quad \iff \quad \alpha_1, \alpha_2 < \tfrac{1}{2},
\]
where $\alpha_i = \mathbb{P}(\widetilde{Z}_i \ne Y)$ is the probability that client $i$'s quantized signal disagrees with the latent truth. Plainly: each client's local training must produce signals that are right more often than wrong---``better than random.''
(Notation note: $\alpha_i$ here is a per-client noise rate; it is distinct from the Dirichlet concentration parameter $\alpha_{\mathrm{dir}}$ used in the empirical sweep below.)

Two structural FL facts make this easy to satisfy even under heavy non-IID:
\begin{enumerate}
    \item \textbf{Shared optimization surface.} Clients descend the \emph{same} loss from the \emph{same} global checkpoint. Gradient magnitudes may differ, but signs tend to agree on most coordinates---the optimum lives in the same direction for everyone.
    \item \textbf{Conditional independence by construction.} Within a round, clients run local SGD without inter-client communication. Conditional on the global checkpoint and the optimum $Y$, signals are independent.
\end{enumerate}

\subsubsection{Mapping from non-IID scenarios to $\alpha_i$}

\begin{table}[h]
\small
\centering
\setlength{\tabcolsep}{4pt}
\renewcommand{\arraystretch}{1.15}
\begin{tabular}{@{}p{3.4cm}p{4.4cm}p{4.4cm}@{}}
\toprule
\textbf{Non-IID scenario} & \textbf{Mechanism} & \textbf{Effect on $\alpha_i$} \\
\midrule
Label skew (Dirichlet $\alpha_{\mathrm{dir}}$) & Imbalanced class gradients on shared layers & Higher on coordinates tied to under-represented classes \\
Quantity skew & Small local $n_i$ raises SGD variance ($\sim 1/\sqrt{n_i}$) & Higher on informative coordinates \\
Feature / domain skew & Local distribution shifts vs.\ global target & Higher where local signal misaligns with optimum \\
Local label noise & Corrupted labels in local loss & Direct linear increase \\
\bottomrule
\end{tabular}
\caption{Each non-IID partition mode raises $\alpha_i$, but the categorical-world condition only fails when $\alpha_i \ge 1/2$---i.e., when local training is actively misleading rather than merely noisy.}
\label{tab:noniid_alpha_mapping}
\end{table}

In every case, the non-IID partition raises noise without flipping signal direction; this rescales $|\Delta|$ while preserving $\mathrm{sign}(\Delta)$.

\subsubsection{Empirical verification}

\textbf{MNIST + LeNet, Dirichlet label skew.} We partition MNIST across 5 clients using Dirichlet sampling at $\alpha_{\mathrm{dir}} \in \{0.1, 0.5, 1, 5, 100\}$ ($0.1$ = severe skew; $100$ = near-IID). Across all $\binom{5}{2} = 10$ client pairs and all 5 non-IID levels (50 pairs total), the empirical $\hat{\Delta}$ satisfies the categorical-world condition. At the harshest skew ($\alpha_{\mathrm{dir}} = 0.1$, global accuracy $56\%$), honest reward is $+0.039$ and any single-client deviation receives $-0.039$ in expectation---a clear separation even under extreme heterogeneity.

\textbf{DistilBERT + AG News.} For $\alpha_{\mathrm{dir}} \ge 0.5$, the condition holds for all 10 client pairs (e.g., $\alpha_{\mathrm{dir}} = 0.5$: honest $+0.150$, attack $-0.150$, accuracy $70.2\%$).

\subsubsection{Where KFCA \emph{should} fail}

The condition has a hard limit. If Client~1 trains only on Class~A and Client~2 only on Class~B, both models become near-constant predictors ($P_1(A \mid y) \approx 1$ for all $y$), so $\mathrm{Cov}_Y[\,\cdot\,] \approx 0$, $\Delta \to 0$, and the KFCA reward collapses to zero. This is the correct behaviour: without overlapping knowledge, mutual validation is impossible. It is an information-theoretic limit shared by Correlated Agreement, Dasgupta--Ghosh, and every peer-prediction mechanism without ground truth---not a KFCA-specific failure mode.

\subsection{Effort and Shirking}
\label{app:effortandshrinking}
The categorical-world structure is robust to partial effort.  
Suppose each client $i$ exerts effort independently with probability \(\eta_i \in [0,1]\).  
When exerting effort (\(e_i^k=1\)), an informative signal is drawn from \(P_i(\cdot \mid Y)\);  
when shirking (\(e_i^k=0\)), an uninformative signal is drawn independently from \(Q_i(\cdot)\),  
which is assumed to be independent of \(Y\) and of other clients' reports.  
Under the standard assumption of conditional independence given \(Y\),  
the observed signal correlation scales linearly with the joint probability that both clients exert effort.

\begin{lemma}[Covariance structure under shirking]
\label{lem:shirking}
Let \(\Delta(a,b)\) denote the signal correlation matrix
\[
    \Delta(a,b) = P(Z_1=a,Z_2=b) - P(Z_1=a)\,P(Z_2=b).
\]
If clients exert effort independently with probabilities \(\eta_1,\eta_2 \in [0,1]\), then
$$
    \Delta(a,b)
    = \eta_1 \eta_2\,
      Cov_{Y \sim \pi}\!\big(P_1(a\mid Y),\,P_2(b\mid Y)\big)
    = \eta_1 \eta_2\, \Delta_{\mathrm{inf}}(a,b),
$$
where \(\pi\) denotes the prior distribution of \(Y\),  
and \(\Delta_{\mathrm{inf}}\) is the correlation matrix under full effort (\(\eta_1=\eta_2=1\)).
\end{lemma}

Hence, shirking uniformly scales down the magnitude of all correlations by \(\eta_1\eta_2\) while preserving their sign pattern:  
\[
    \Delta(a,a) > 0, \qquad \Delta(a,b) < 0, \quad a \neq b.
\]
If either client never exerts effort (\(\eta_i=0\)), then \(\Delta(a,b)=0\) for all \(a,b\),  
meaning a completely uninformative participant contributes no expected reward.

\subsection{Shapley Value}
\label{app:shapley_value}

\subsubsection{Definition and axioms}
The Shapley value assigns each client a unique score according to their marginal contribution to the overall outcome as a comprehensive assessment of each client's contribution. Formally, in a cooperative game $(N, v)$, where $N$ represents the set of clients and $v: 2^N \rightarrow \mathbb{R}$ is the characteristic function mapping a coalition of clients to a real value, the Shapley value $\phi_i$ for client $i \in N$ is defined as:

\begin{equation}
\phi_i(v)=\frac{1}{N!} \sum_{\Pi}\left[v\left(S_i^{\Pi} \cup\{i\}\right)-v\left(S_i^{\Pi}\right)\right]
\label{eq:shapley}
\end{equation}
where the sum ranges over all $N!$ permutations $\Pi$ of the clients. $S_{i}^{\Pi}$ is the set of clients in $N$ which precede $i$ in the respective permutation $\Pi$.\footnote{Let there be four clients 1, 2, 3, and 4. Given the permutation $\Pi=4213$, then $S_3^{\Pi} = \{4, 2, 1\}$.}  $v(S)$ denotes the contribution value (e.g. accuracy or loss in FL) of a given subset of clients $S \in N$.

The Shapley consists of four axioms of fairness that are believed a to be a fair allocation measure:

\begin{enumerate}
    \item \textbf{Efficiency/Pareto Optimality:} The total payout to all players in the coalition \(N\) should be equal to the total value generated by the coalition, i.e., 
    \[
    \sum_{i \in N} \phi_i(v) = v(N).
    \]
    This ensures that no value is left unallocated.
    
    \item \textbf{Symmetry:} If two players \(i\) and \(j\) contribute equally to all coalitions they are part of, they should receive the same payoff, i.e., 
    if $v(S \cup \{i\}) = v(S \cup \{j\})$ for all $S \subseteq N \setminus \{i, j\}$  then $\phi_i(v) = \phi_j(v)$.
    This ensures equal pay for equal contribution.
    
    \item \textbf{Null Player:} If a player \(i\) does not contribute any value to any coalition it is part of, the player should receive a payoff of 0, i.e., 
   if $v(S \cup \{i\}) = v(S)$ for all $S \subseteq N \setminus \{i\}$, then $\phi_i(v) = 0$.
    This axiom ensures that players that do not contribute to any coalition receive nothing.
    
    \item \textbf{Additivity:} If \(u\) and \(v\) are characteristic functions, then the payoff for a player under the sum of these games is equal to the sum of the player's payoffs under each of these games, i.e., 
    \[
    \phi(u + v) = \phi(u) + \phi(v).
    \]
    This axiom ensures the Shapley value is consistent when games are combined.
\end{enumerate}

\subsubsection{Shapley Value in Federated Learning}
Let $ \mathbf{A}$ be a learning algorithm and let $\theta_{S}$ be trained on the combined dataset $ D^{S} := \{X^S, Y^S\}$ where 
$ X^S = \{ X^i \cup X^j \cup X^k \cup \dots \}$ is the feature space and $ Y^S = \{ Y^i \cup Y^j \cup Y^k \cup \dots \}$ label space of clients $S={i, j, k \dots} \subseteq N$ on a testset $D^{test}$, then

\begin{equation}
\label{eq:marginal_contribution}
v(S) = V(\mathbf{A}(\theta_{init}, D^S), D^{test})
\end{equation}

for $S = {i,j,k } \in N$. 

\subsubsection{Example: Shapley value in FL with three clients}

Consider a FL system with three clients, denoted by $N = \{1, 2, 3\}$.  The model is initialized to classify MNIST with 10 labels and is not pretrained, resulting in an initial accuracy of $v(\emptyset)= 0.1$. Let's assume the following example accuracies for each possible subset of client $S$:

\begin{table}[t]
\caption{Accuracy $v(S)$ for Coalition $S$}
\centering
\begin{tabular}{lr}
    \toprule
    $\textbf{S}$ & $\textbf{v(S)}$ \\
    \midrule
    $\{\emptyset\}$ & 0.1 \\
    $\{1\}$ & 0.7 \\
    $\{2\}$ & 0.75 \\
    $\{3\}$ & 0.8 \\
    $\{1,2\}$ & 0.85 \\
    $\{1,3\}$ & 0.9 \\
    $\{2,3\}$ & 0.95 \\
    $\{1,2,3\}$ & 0.98 \\
    \bottomrule
\end{tabular}

\label{table:Shapley_example}
\end{table}

To compute the Shapley values for each client, we need to evaluate all possible coalitions and their marginal contributions. Following equation~\ref{eq:shapley} to calculate the Shapley value with the $v(S)$ values given in Table~\ref{table:Shapley_example}. The Shapley values for clients 1, 2, and 3 are

\begin{align}
  \phi_{1} &= \frac{1}{6}\left[\underbrace{0.6}_{0.7 - 0.1} + \frac{1}{2}\cdot\underbrace{0.1}_{0.85 - 0.75} + \frac{1}{2}\cdot\underbrace{0.1}_{0.9 - 0.8} + \underbrace{0.03}_{0.98 - 0.95}\right], \nonumber \\
  \phi_{2} &= \frac{1}{6}\left[\underbrace{0.65}_{0.75 - 0.1} + \frac{1}{2}\cdot\underbrace{0.15}_{0.85 - 0.7} + \frac{1}{2}\cdot\underbrace{0.15}_{0.95 - 0.8} + \underbrace{0.08}_{0.98 - 0.9}\right], \nonumber \\
  \phi_{3} &= \frac{1}{6}\left[\underbrace{0.7}_{0.8 - 0.1} + \underbrace{0.2}_{0.9 - 0.7} + \frac{1}{2}\cdot\underbrace{0.2}_{0.95 - 0.75} + \frac{1}{2}\cdot\underbrace{0.13}_{0.98 - 0.85}\right]. \nonumber
\end{align}

\subsubsection{Limitations of Shapley value}
Shapley value calculation in FL faces two major limitations, computational overhead and consensus on a testset.

Computing the Shapley value for each participant in FL is both computationally demanding and practically infeasible when the number of clients is large. The process necessitates (i) retraining and (ii) reevaluating the model for $N!$ subsets $S$ to obtain the accuracies $v(S)$ for all possible permutations of clients. For instance, for 10 clients, Shapley value computation would require retraining the model $10! = 3,628,800$ times. Given these challenges, existing literature primarily focuses on maintain computational feasibility while approximating the true Shapley value, through Monte Carlo estimation of permutations \cite{shapleyeval, shapleyMLicml, gtgShapley} and model reconstruction techniques 
\cite{ProfitAllocationFL, wei2020efficient, wang2020principled, gtgShapley}, and truncation \cite{gtgShapley}.

\begin{enumerate}
    \item \textit{Avoiding retraining}: Shapley value computation can be expedited by reconstructing submodels $\theta^{S}$ from gradients instead of retraining the model on $D^{S}$ for every subset $S$ (equation~\ref{eq:marginal_contribution}), according to \begin{equation}
v(S)=V\left((\theta^{global}+\sum_{i \in S} \frac{\left|D^i\right|}{\left|D^S\right|} \Delta\theta^i), D^{pub}\right),
\end{equation}
omitting computationally expensive training. However, the evaluating entity has to have access to gradients of all participating clients $i \in N$.

\item \textit{Avoiding low-impact calculations}: The marginal contribution of a client heavily depends on its position in the permutation $\Pi$. As the marginal utility usually decreases the later it joins the subset, \textit{truncating} the Shapley value calculation if its marginal contribution is not significantly different from the previous one, e.g. $|v(D^{N} )- v(S^{\Pi}_{i})| \leq \epsilon $ significantly speeds up the calculation process.
\item \textit{Reducing subsets S}
To further improve the efficiency of Shapley value computation, instead of going over all possible permutations, Monte Carlo estimation \cite{MonteCarloSimulation} can be applied to randomly sample permutations $\Pi$ \cite{shapleyMLicml,AISTATSshapleyML,ShapleyFLSpringer, gtgShapley}, and then calculate the expected SV according to 
\begin{equation}
     \phi_i=\mathbb{E}_{\pi \sim \uppi}\left[V\left(S_i^{\pi} \cup\{i\}\right)-V\left(S_i^{\pi}\right)\right]
\end{equation}

where $\uppi$ is the uniform distribution over all permutations $N!$. 
\end{enumerate}

Despite these optimization techniques, it still requires $3N$ Monte Carlo simulations for sufficient convergence, which equates to $3N$ model reconstructions and inferences on the testset $D^{pub}$ \cite{shapleyMLicml}. In addition to the computational overhead, SV requires a central authority that (i) determines the testset on which the SV is being evaluated and (ii) is entrusted to calculate the contributions and distribute the rewards fairly.

\subsubsection{Scalability of KFCA: Wall-Clock Measurements}
\label{app:scalability}

In contrast to the cost analysis of Shapley-value estimators above, KFCA itself imposes only $\mathcal{O}(npm)$ work per round (Sec.~\ref{KFCA}). Figure~\ref{fig:time} reports the corresponding wall-clock measurements: computation time scales linearly with the number of clients $n$ for a fixed number of peers per client $p$, and the choice of $p$ trades runtime for the standard error $\sigma/\sqrt{p}$ of the per-pair score. KFCA-QP was run on a simple CNN with 5{,}280 parameters and KFCA-D on a 10{,}000-unit dataset.

\begin{figure}[h]
    \centering
    \includegraphics[width=.5\linewidth]{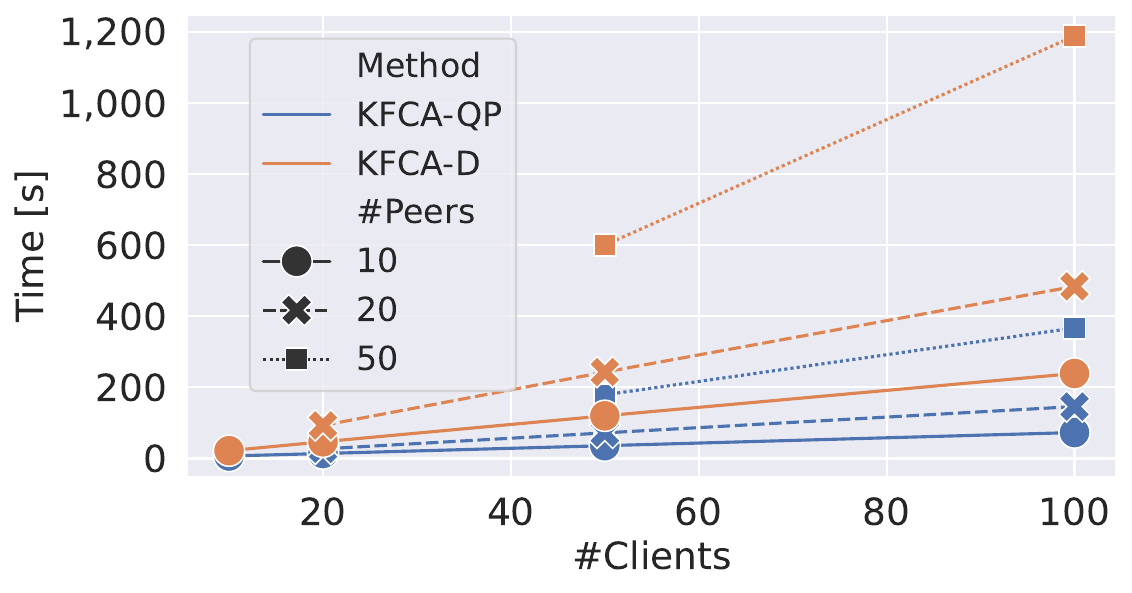}
    \caption{\small Computation time vs.\ \#Clients and \#Peers. KFCA-QP on a simple CNN with 5{,}280 parameters and KFCA-D on a 10{,}000-unit dataset, measured on a single Apple~Silicon M2~Pro (16~GB RAM).}
    \label{fig:time}
\end{figure}

\subsubsection{Experiments: Comparison Analysis Details}
\label{app:comparisonmethodsSV}

\begin{enumerate}
    \item \textbf{Shapley Value:} Original SV calculation for FL client contribution evaluation, based on Equation~\ref{eq:shapley}.
    \item \textbf{GTG-Shapley:} \cite{gtgShapley,GTGAAAI_Medical} Utilizes sub-model reconstruction with gradient updates from clients, guided sampling, in-round truncation of client permutations, and between-round truncation that drops an entire round of SV calculation if the remaining marginal utility or accuracy gain is small.
    \item \textbf{TMC-Shapley:} \cite{shapleyMLicml} Estimates the Shapley values by employing Monte Carlo sampling of permutations and selectively truncating the sub-model training and evaluations of irrelevant FL clients.
    \item \textbf{Group Testing:} \cite{AISTATSshapleyML} Uses Shapley differences instead of Shapley values, with the original Shapley value derived from the Shapley differences by solving a feasibility problem.
    \item \textbf{MR:} \cite{ProfitAllocationFL} In each round of the FL process, reconstructs the model of a subset of participants using their gradient updates. The final SV for a participant is obtained by summing up their SVs across all rounds.
    \item \textbf{Fed-SV:} \cite{wang2020principled} A group testing-based estimation approach. Performance of subsets used for estimating Shapley differences is evaluated on a sub-model reconstructed using participants' model parameters, and Shapley values are independently estimated in each round and subsequently aggregated.
    \item \textbf{TMR:} \cite{wei2020efficient} A gradient update-based method for SV calculation, incorporating a decay factor to include SV from previous rounds and a truncation factor to omit unimportant sub-model reconstructions.
    \item \textbf{Correlated Agreement (CA-QP):} \cite{CA_HONGTAO} Calculates the delta matrix based on all quantized model parameters of clients.
    \item \textbf{CA-D:} \cite{LIU_CA} Calculates the delta matrix using all labels in the public dataset.
    \item \textbf{KFCA-QP:} Simplified CA based on quantized model parameters of clients.
    \item \textbf{KFCA-D:} Simplified CA on the Test Dataset, assuming the delta matrix is the identity matrix and rewards are based on predictions on the test set.
    \item \textbf{MC KFCA-QP:} Monte Carlo version of KFCA-QP, randomly choosing \(X\) parameters \(Y\) times out of all parameters and averaging the rewards.
\end{enumerate}

\subsubsection{Experiments: Additional Details}
\label{app:experiments_details}

\paragraph{MNIST Shapley-value comparison setup (details).}
We follow the federated evaluation protocol of \cite{wei2020efficient} and use a simple CNN (21,840 parameters) trained on MNIST \cite{MH_MNIST}. The dataset is partitioned into five equal silos, and we form 10 clients by sampling client pairs from each silo (clients 1--2 from silo 1, clients 3--4 from silo 2, etc.).

\paragraph{Client heterogeneity/noise cases.}
We evaluate five settings:
\begin{enumerate}
    \item \textbf{Case 1 (i.i.d., same size):} Each client receives 10,840 images sampled to be balanced across digits.
    \item \textbf{Case 2 (label skew, same size):} Each client still has 10,840 images. Clients 1--2 have 80\% digits \texttt{1}/\texttt{2} and 20\% spread over remaining digits; clients 3--4 are skewed to \texttt{3}/\texttt{4}, etc.
    \item \textbf{Case 3 (size skew, i.i.d. labels):} All clients have the same label distribution, but different dataset sizes with ratios: 10\% (clients 1--2), 15\% (3--4), 20\% (5--6), 25\% (7--8), 30\% (9--10).
    \item \textbf{Case 4 (label noise, same size):} We flip labels at increasing rates: 0\% (clients 1--2), 5\% (3--4), 10\% (5--6), 15\% (7--8), 20\% (9--10).
    \item \textbf{Case 5 (feature noise, same size):} We add Gaussian feature noise at increasing rates: 0\% (clients 1--2), 5\% (3--4), 10\% (5--6), 15\% (7--8), 20\% (9--10).
\end{enumerate}

\paragraph{Methods compared.}
We compare KFCA to CA variants and to efficient Shapley value estimators; a brief description of each baseline is provided in Subsubsection~\ref{app:comparisonmethodsSV}.

\begin{figure*}[h]
  \centering
  \begin{subfigure}[b]{0.33\textwidth}
    \centering
    \includegraphics[width=\textwidth]{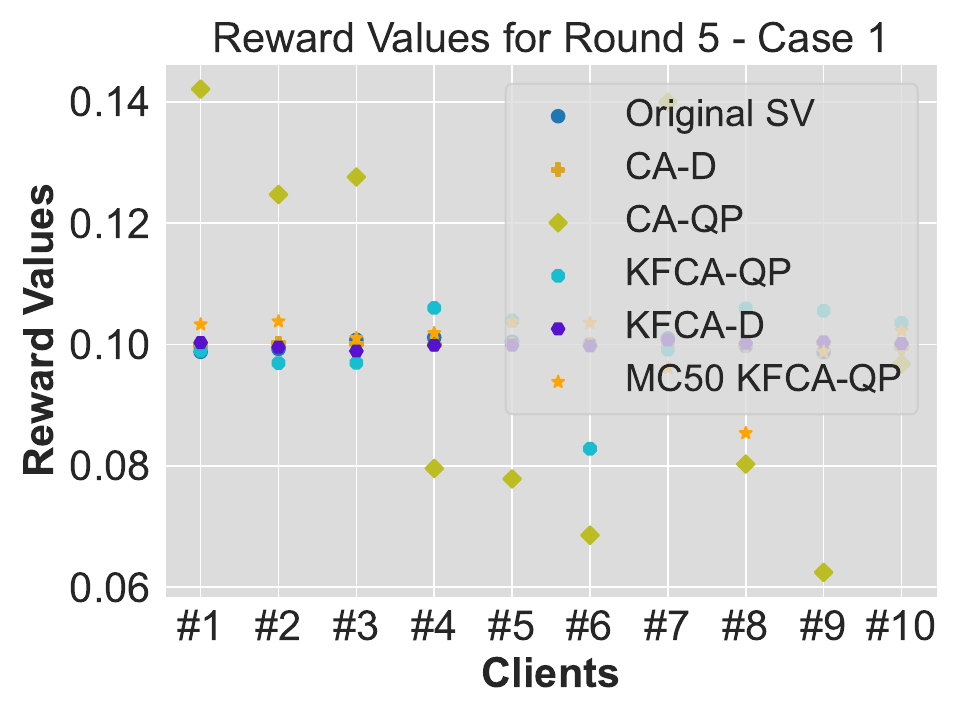}
    \caption{Case 1}
    \label{fig:graph1}
  \end{subfigure}
  \begin{subfigure}[b]{0.33\textwidth}
    \centering
    \includegraphics[width=\textwidth]{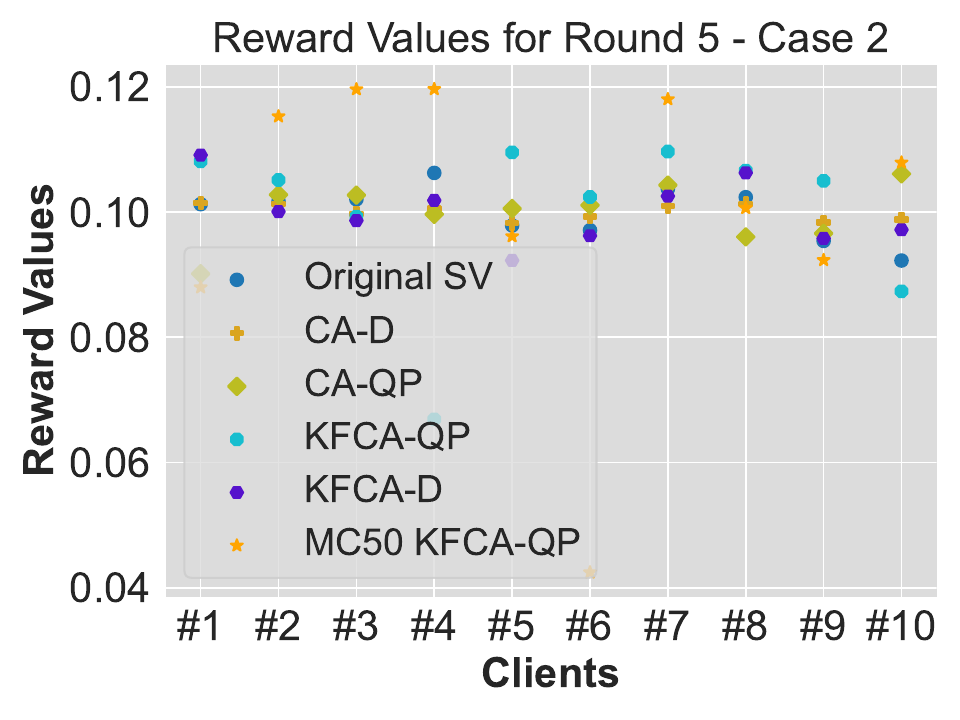}
    \caption{Case 2}
    \label{fig:graph2}
  \end{subfigure}
  \begin{subfigure}[b]{0.33\textwidth}
    \centering
    \includegraphics[width=\textwidth]{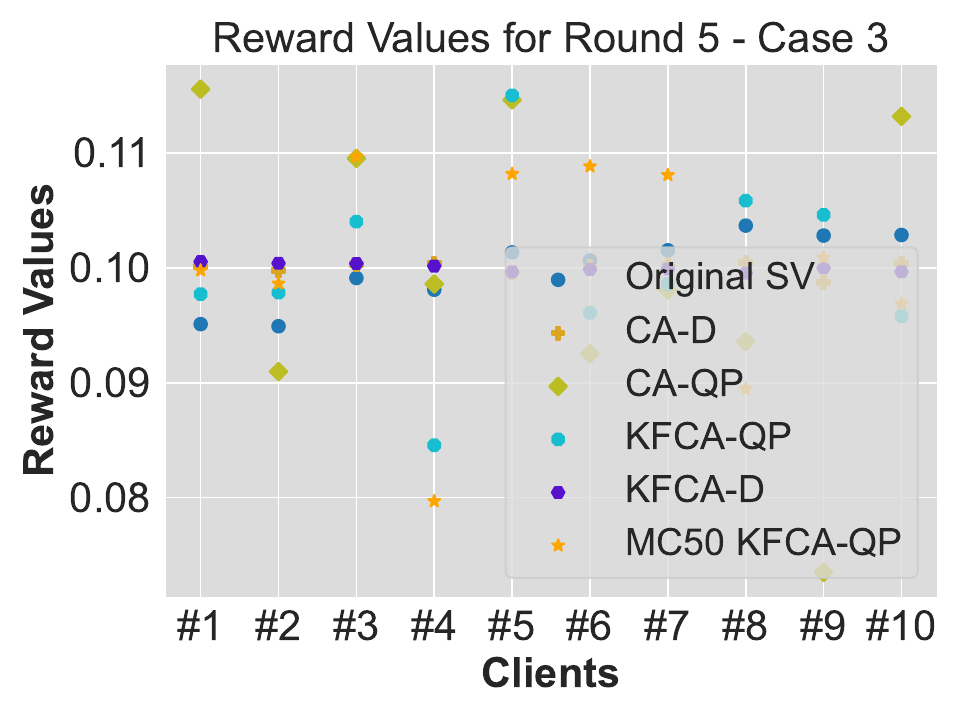}
    \caption{Case 3}
    \label{fig:graph3}
  \end{subfigure}
  \begin{subfigure}[b]{0.33\textwidth}
    \centering
    \includegraphics[width=\textwidth]{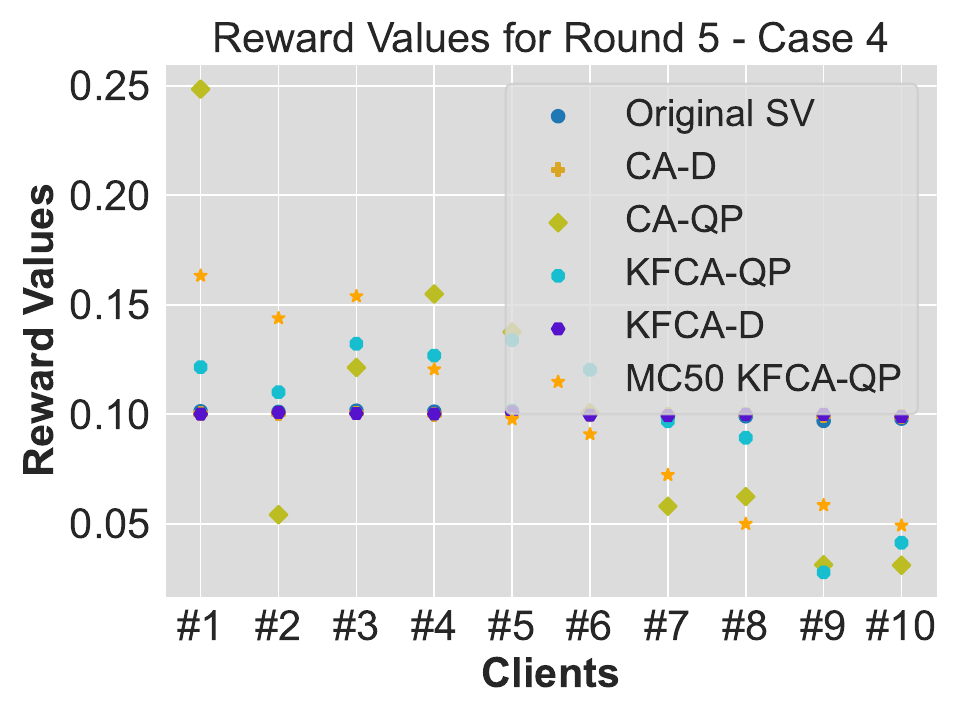}
    \caption{Case 4}
    \label{fig:graph4}
  \end{subfigure}
  \begin{subfigure}[b]{0.33\textwidth}
    \centering
    \includegraphics[width=\textwidth]{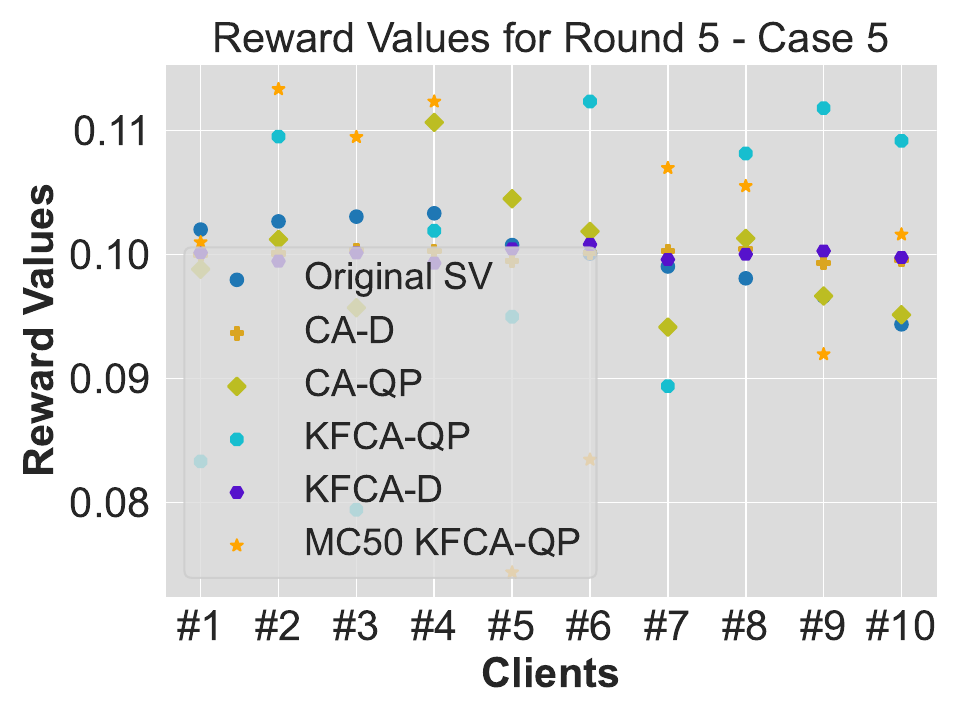}
    \caption{Case 5}
    \label{fig:graph5}
  \end{subfigure}
  \caption{Reward distribution for individual clients}
  \label{app:rewarddistributionsingleclient}
\end{figure*}

\subsection{KFCA-QP for Federated LLM Fine-Tuning}
\label{app:llm_finetune}
\subsubsection{Motivation: Federated LLM Fine-Tuning}

Large language models (LLMs) now deliver state-of-the-art performance across a wide range of tasks and domains \cite{gaoflowertune_2025}. 
Yet this progress faces two structural bottlenecks: modern LLM development remains heavily dependent on vast public corpora, and the supply of high-quality public text may be exhausted within a few years \cite{spiceworks_2024_data,villalobos_2022_runout}. 
At the same time, the most valuable domain knowledge—particularly in healthcare and finance—often sits inside institutions where privacy and security requirements make data centralization difficult, compounded by practical barriers in storage, transfer, and communication infrastructure \cite{gaoflowertune_2025}. 
Federated learning offers a practical alternative: it enables collaborative fine-tuning of pre-trained LLMs across decentralized data holders without sharing raw data, expanding access to sensitive datasets while preserving privacy and supporting domain-specialized models \cite{ye_2024_openfedllm,ye_2024_fedllmbench,gaoflowertune_2025}.

Federated LLM fine-tuning typically communicates only lightweight adapters (e.g., LoRA~\cite{hu2022lora} /DoRA~\cite{liu2024dora}), updating \(W\in\mathbb{R}^{d\times k}\) via \(W+\Delta W\) with \(\Delta W=BA\), \(A\in\mathbb{R}^{r\times k}\), \(B\in\mathbb{R}^{d\times r}\), \(r\ll\min(d,k)\) (and DoRA further decomposes weights into magnitude and direction) usually with no shared evaluation/test set available.

We validate KFCA-QP on the FlowerTune LLM Leaderboard~\cite{flowertune_leaderboard_website}, which— to the best of our knowledge—constitutes a first-of-its-kind, public, cross-domain benchmark suite for federated fine-tuning of LLMs.

\begin{figure}[!htbp]
\centering
\includegraphics[width=\textwidth]{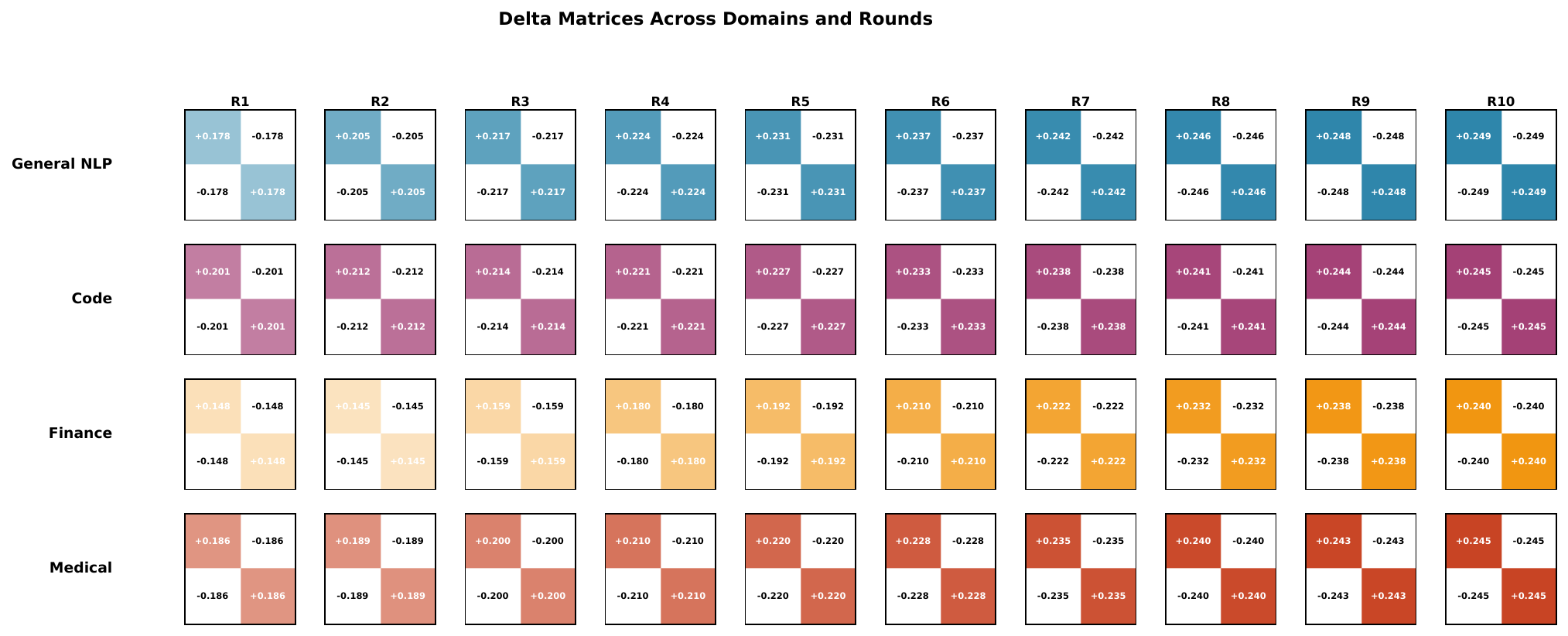}
\caption{Empirical $\hat{\Delta}$ matrices for quantized adapter updates across domains (rows) and rounds (columns). Each cell is a $2\times2$ matrix over $\{-1,+1\}$. All matrices satisfy the categorical-world condition: $\Delta(-1,-1),\Delta(+1,+1)>0$ and $\Delta(-1,+1),\Delta(+1,-1)<0$.}
\label{fig:delta_matrices}
\end{figure}

\subsubsection{Experimental Setup}
We instantiate KFCA-QP on the four winner configurations (General NLP, Finance, Medical, Code) from the FlowerTune LLM Leaderboard~\cite{gao2025flowertune}. Following FlowerTune, we operate in a parameter-efficient tuning regime: the base LLM is frozen and clients train only adapter parameters (LoRA/DoRA) locally, communicating only adapter updates for FedAvg aggregation each round. Table~\ref{tab:flowertune_setup} summarizes training and federation hyperparameters, and Table~\ref{tab:flowertune_eval_results} reports the evaluation suites/metrics together with downstream performance after 10 rounds (including total communication).

\begin{table}[t]
\centering
\caption{FlowerTune federated adapter-tuning setup (winner configurations). All runs use FedAvg, 4-bit quantization, sequence length 512, and 3 local epochs (capped at 10 local steps per round). Client splits are i.i.d.\ and equally sized.}
\label{tab:flowertune_setup}
\small
\setlength{\tabcolsep}{4pt}
\renewcommand{\arraystretch}{1.15}
\resizebox{\columnwidth}{!}{%
\begin{tabular}{lllccccccc}
\toprule
\textbf{Domain} & \textbf{Base model} & \textbf{Train data} & \textbf{\#Clients} & \textbf{frac\_fit} & \textbf{Per rnd} & \textbf{Adapter} & \textbf{Targets} & \textbf{Batch$\times$accum} & \textbf{LR (max/min)} \\
\midrule
General NLP & InternLM3-8B-Instruct & alpaca-gpt4~\cite{peng2023instruction} & 20 & 0.1 & 2 & LoRA ($r$=16, $\alpha$=32) & q/k/v/o proj & $1\times4$  & $5{\times}10^{-5}/10^{-6}$ \\
Finance     & Qwen2.5-7B             & fingpt-sentiment-train~\cite{yang2023fingpt} & 20 & 0.1 & 2 & LoRA ($r$=32, $\alpha$=64) & all (proj+MLP) & $16\times1$ & $5{\times}10^{-4}/5{\times}10^{-5}$ \\
Medical     & Llama3.1-Aloe-Beta-8B  & medical-flashcards~\cite{han2023medalpaca} & 20 & 0.1 & 2 & LoRA ($r$=32, $\alpha$=128) & q\_proj, v\_proj & $16\times2$ & $10^{-4}/10^{-5}$ \\
Code        & Qwen3-8B               & code-alpaca-20k~\cite{chaudhary2023code} & 10 & 0.2 & 2 & LoRA+DoRA ($r$=8, $\alpha$=16) & all (proj+MLP) & $2\times4$  & $5{\times}10^{-5}/5{\times}10^{-6}$ \\
\bottomrule
\end{tabular}%
}
\end{table}

\begin{table}[t]
\centering
\caption{FlowerTune evaluation suites and downstream performance after 10 federated rounds (winner configurations). MQA = multiple-choice question answering. \textbf{Comm.} is total communicated adapter payload (upload+download) aggregated over all rounds.}
\label{tab:flowertune_eval_results}
\small
\setlength{\tabcolsep}{3.5pt}
\renewcommand{\arraystretch}{1.15}

\begin{tabular}{l p{0.32\columnwidth} c cc p{0.30\columnwidth}}
\toprule
\textbf{Domain} &
\textbf{Eval datasets} &
\textbf{Metric} &
\textbf{Comm.} &
\textbf{Avg.} &
\textbf{Suite scores (labeled)} \\
\midrule
General NLP &
MMLU (STEM/Hum./Soc.\ Sci.)~\cite{hendrycks2021measuring} &
Acc. &
2.9\,GB &
69.25 &
STEM 66.13; Soc.\ Sci.\ 80.76; Hum.\ 60.87 \\
\addlinespace[1pt]

Finance &
FPB~\cite{malo2014good}, FIQA~\cite{maia2018fiqa}, TFNS~\cite{yang2020finbert} &
Acc. &
30.1\,GB &
84.58 &
FPB 85.89; FIQA 83.22; TFNS 84.63 \\
\addlinespace[1pt]

Medical &
PubMedQA~\cite{jin2019pubmedqa}, MedMCQA~\cite{pal2022medmcqa}, MedQA~\cite{jin2021medqa}, CareQA~\cite{arias-duart-etal-2025-automatic} &
Acc. &
2.0\,GB &
63.57 &
PubMedQA 74.80; MedMCQA 55.39; MedQA 59.31; CareQA 64.79 \\
\addlinespace[1pt]

Code &
MBPP~\cite{austin2021program}, HumanEval~\cite{chen2021evaluating}, MultiPL-E (JS/C++)~\cite{cassano2023multiple} &
Pass@1 &
3.5\,GB &
65.27 &
MBPP 56.20; HumanEval 70.73; MultiPL-E JS 68.32; C++ 65.84 \\
\bottomrule
\end{tabular}
\end{table}

\paragraph{KFCA-QP signal construction.}
After each round $t$, client $i$ uploads its adapter update $\Delta\theta_i^{t} \in \mathbb{R}^{d}$, where $d$ is the number of trainable LoRA/DoRA adapter parameters. We construct categorical signals via 1-bit quantization:
\[
\widetilde{Z}_i^{t}[p] = \mathrm{sign}(\Delta\theta_i^{t}[p]) \in \{-1,+1\}, \quad \forall p \in [d],
\]
treating each parameter coordinate as a task. KFCA-QP rewards are computed from pairwise agreement between sampled clients' quantized reports.

\subsubsection{Results: Categorical-World Condition and Incentive Compatibility}

Figure~\ref{fig:delta_matrices} visualizes the empirical delta matrices for quantized LoRA/DoRA adapter updates across all four domains and 10 rounds. The diagonal entries (colored by domain) are consistently positive while off-diagonal entries are negative, confirming $\text{Sign}(\hat{\Delta}) = \mathbb{I}$---the categorical-world condition required for KFCA-QP across all four domains and every single round.

\begin{figure}[!htbp]
\centering
\includegraphics[width=\textwidth]{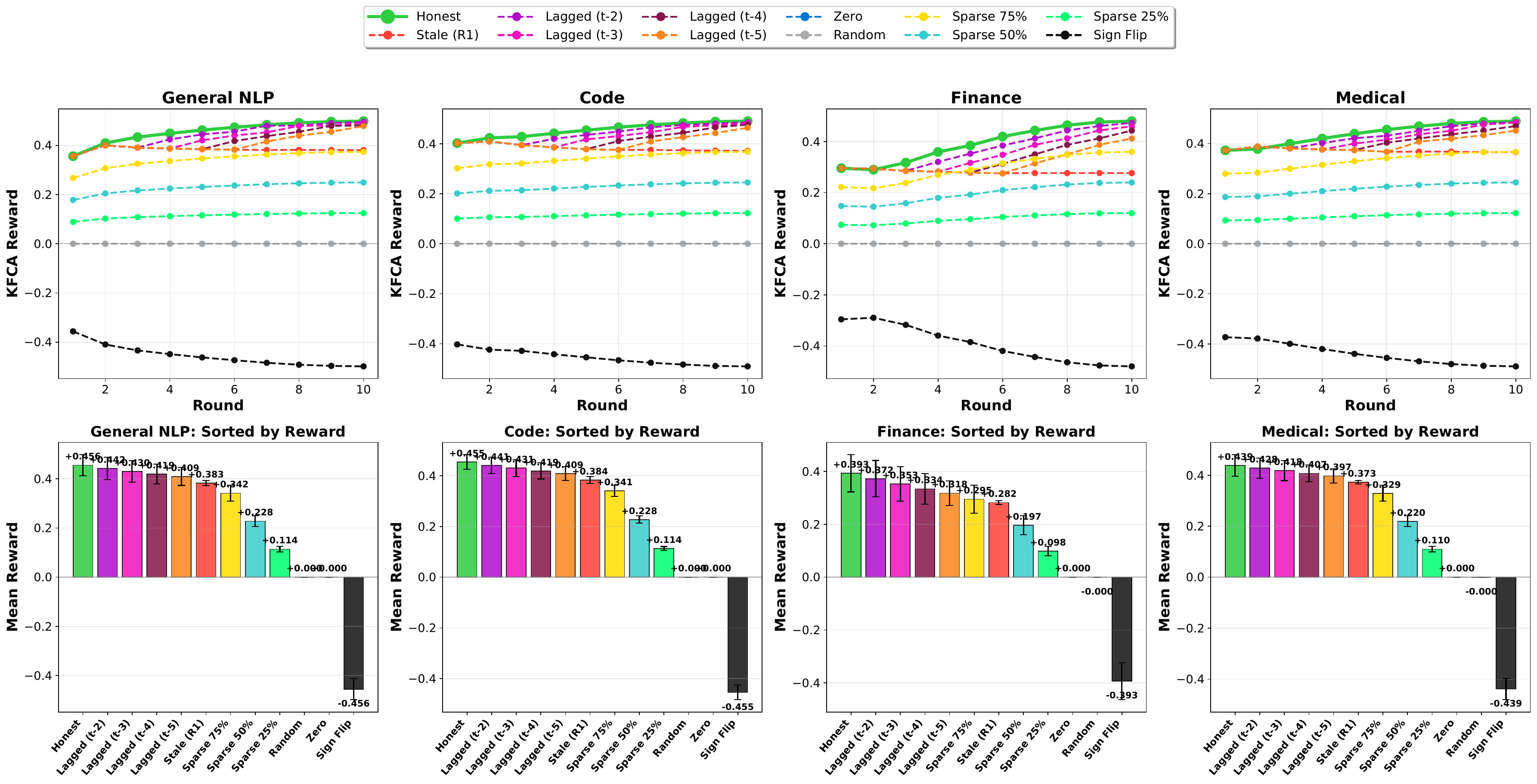}
\caption{KFCA-QP rewards under various attack strategies across four federated LLM fine-tuning domains. \textbf{Top row:} Reward trajectories over 10 communication rounds. \textbf{Bottom row:} Mean rewards sorted by magnitude. Honest reporting consistently achieves the highest reward, while adversarial strategies (sign flip, zero, random) receive zero or negative rewards. Lagged and sparse attacks receive intermediate rewards proportional to their deviation from honest behavior, demonstrating KFCA-QP's incentive compatibility.}
\label{fig:attack_analysis}
\end{figure}

\paragraph{1. Temporal Manipulation Attacks.}
\begin{itemize}
    \item \textbf{Stale Update (R1):} The client submits the round-1 adapter update $\Delta\theta_{s}^{1}$ for all subsequent rounds, ignoring local training thereafter.
    \item \textbf{Lagged Update (t-$k$):} The client submits $\Delta\theta_{s}^{t-k}$ in round $t$, where $k \in \{2,3,4,5\}$.
\end{itemize}

\paragraph{2. Free-Rider Attacks.}
\begin{itemize}
    \item \textbf{Zero Attack:} The client submits an all-zero update $\Delta\theta_{s}^{t}=\mathbf{0}$. After sign quantization, this maps to a constant report, breaking correlation with honest peers.
    \item \textbf{Random Attack:} The client submits random noise $\Delta\theta_{s}^{t}\sim \mathcal{N}(0,\sigma^{2}I)$ scaled to match the honest update's variance.
\end{itemize}

\paragraph{3. Partial Manipulation Attacks.}
\begin{itemize}
    \item \textbf{Sparse Attack ($p$\%):} The client reports honestly for $p$\% of coordinates and replaces the remaining $(100-p)$\% with random values. We test $p\in\{25,50,75\}$.
    \item \textbf{Sign Flip Attack:} The client inverts all signs, i.e., $\Delta\theta_{s}^{t}\leftarrow -\Delta\theta_{s}^{t}$.
\end{itemize}

\paragraph{Results.}
Figure~\ref{fig:attack_analysis} reports the KFCA-QP rewards across all four domains using peer-to-peer comparison (client vs.\ another honest client from the same round). The results demonstrate strong incentive compatibility:

\begin{itemize}
    \item \textbf{Honest reporting dominates:} Honest participation consistently achieves the highest reward across all domains (approximately $+0.21$ to $+0.45$).
    
    \item \textbf{Complete attacks severely penalized:} Sign flip receives strongly negative rewards ($\approx -0.37$ to $-0.46$), confirming that adversarial manipulation is punished. Zero and random attacks yield zero expected reward, as predicted by theory.
    
    \item \textbf{Temporal attacks partially penalized:} Lagged updates receive diminishing rewards as the lag increases (t-2 $>$ t-3 $>$ t-4 $>$ t-5 $>$ stale), demonstrating that KFCA-QP detects staleness proportionally to deviation from current consensus.
    
    \item \textbf{Sparse attacks interpolate:} Rewards scale approximately linearly with the honest fraction---Sparse 75\% achieves roughly 50\% of honest reward, Sparse 50\% achieves 25\%, and Sparse 25\% approaches zero. This confirms KFCA-QP rewards truthfulness proportionally.
\end{itemize}

These results validate KFCA-QP's incentive compatibility: the mechanism rewards honest participation, penalizes adversarial behavior, and appropriately scales penalties for partial deviations. Importantly, no attack strategy achieves higher reward than honest reporting, eliminating rational incentives for strategic manipulation in federated LLM fine-tuning.

\subsection{Decentralized FL: KFCA on Blockchain}
\label{app:blockchain}

\paragraph{Motivation.}
Although FL is decentralized at the data level, standard deployments still rely on a central server to (i) decide who participates, (ii) act as a single point of failure, and (iii) arbitrate and execute payments in incentivized FL.
A blockchain can replace this trusted coordinator with a public, append-only ledger and \emph{immutable} smart contracts that enforce protocol rules and payouts without relying on any single party.
Modern platforms (e.g., Ethereum, Hyperledger) provide programmable contracts whose state transitions are transparent and auditable, yet tamper-resistant under decentralized consensus.

\paragraph{Why KFCA fits the on-chain setting.}
Smart contracts operate under strict computation and storage budgets (every validator re-executes transactions).
This clashes with reward mechanisms that require heavy evaluation, repeated retraining, or estimating global report distributions.
KFCA is well matched to this environment because it is (i) lightweight, (ii) one-shot (rewards from a single matched comparison), and (iii) distribution-knowledge free.
Consequently, KFCA enables \emph{verifiable, real-time} contribution scoring and automated payouts with minimal on-chain burden.

\subsubsection{KFCA on Blockchain: Design Draft}
KFCA is particularly promising for \emph{decentralized incentivized FL} because its scoring rule is simple enough to be enforced by smart contracts.
In contrast, existing decentralized incentive schemes often either (i) require specialized, application-specific chains, or (ii) face severe scalability limits due to the overhead of decentralized reward computation \cite{LeonSP, LW1_witt2021rewardbased}.
Figure~\ref{fig:KFCA_X_Blockchain} sketches a practical integration of KFCA (e.g., KFCA-D; similarly KFCA-QP) on an EVM-like platform:

\begin{enumerate}
    \item \textbf{Client registration (identity \& rules).}
    Clients register a public address in a smart contract and authenticate via signature.
    Optionally, clients lock a stake (Sybil resistance / slashable misbehavior) or registration is permissioned (known consortium participants).

    \item \textbf{Local training (off-chain).}
    Clients train locally under the agreed FL specification (model, optimizer, rounds, etc.).
    This step remains off-chain.

    \item \textbf{Public evaluation signal (client-side).}
    Given a public reference set $X^{pub}$, each client produces a report:
    predicted labels (KFCA-D) or a quantized update / adapter signature (KFCA-QP).
    Figure~\ref{fig:KFCA_X_Blockchain} illustrates the KFCA-D instance.

    \item \textbf{Commitment (privacy-preserving commit--reveal).}
    Raw reports are not posted on-chain.
    Instead, clients store the report off-chain (e.g., IPFS \cite{benet2014ipfs}) and commit to it on-chain via
    \[
        c_i = H(\text{CID}_i \,\|\, s_i),
    \]
    where $\text{CID}_i$ is the IPFS content identifier and $s_i$ is a client-chosen salt.
    The commit is signed and recorded, preventing post-hoc manipulation while hiding the report from brute-force guessing.

    \item \textbf{Unbiased pairing via VRF (oracle randomness).}
    KFCA requires pairing clients.
    To avoid manipulable randomness, a decentralized VRF (e.g., Chainlink VRF) outputs a verifiable random seed used to sample a pair of committed clients uniformly at random.

    \item \textbf{Reveal \& verification.}
    The selected clients reveal $(\text{CID}_i, s_i)$.
    The contract verifies $H(\text{CID}_i \,\|\, s_i)=c_i$ before accepting the report pointer.
    This guarantees integrity: the revealed report matches the earlier commitment.

    \item \textbf{On-chain KFCA scoring (randomized set selection).}
    Using VRF-derived randomness, the contract samples the index sets required by KFCA (bonus set $M_b$ and penalty sets $M_1, M_2$),
    fetches the corresponding report entries from the off-chain object (or via an agreed data-availability interface),
    and computes the KFCA score for the paired clients transparently.

    \item \textbf{Automatic payout (and optional slashing).}
    Scores are mapped to payments by the protocol-defined transfer rule.
    Funds are released atomically to client addresses; if staking is enabled, provable protocol violations can trigger slashing.
\end{enumerate}

\paragraph{Prototype status and scope.}
We implemented a proof-of-concept KFCA contract in Solidity for the EVM\footnote{We will release the repository after acceptance to preserve double-blind review.}.
A full decentralized FL stack is beyond this paper; the above is a \emph{design draft} that highlights the minimal building blocks needed to make KFCA verifiable and incentive-compatible on-chain.
Key engineering choices left for future work include:
(i) how to integrate aggregation (e.g., FedAvg) without reintroducing a trusted coordinator,
(ii) who instantiates tasks and funds rewards (single sponsor vs.\ DAO/consortium),
and (iii) robustness evaluation under realistic adversaries (Sybil attacks, censorship, data-availability failures, and report-privacy leakage).
\begin{figure*}[t]
    \centering
    \includegraphics[width=\textwidth]{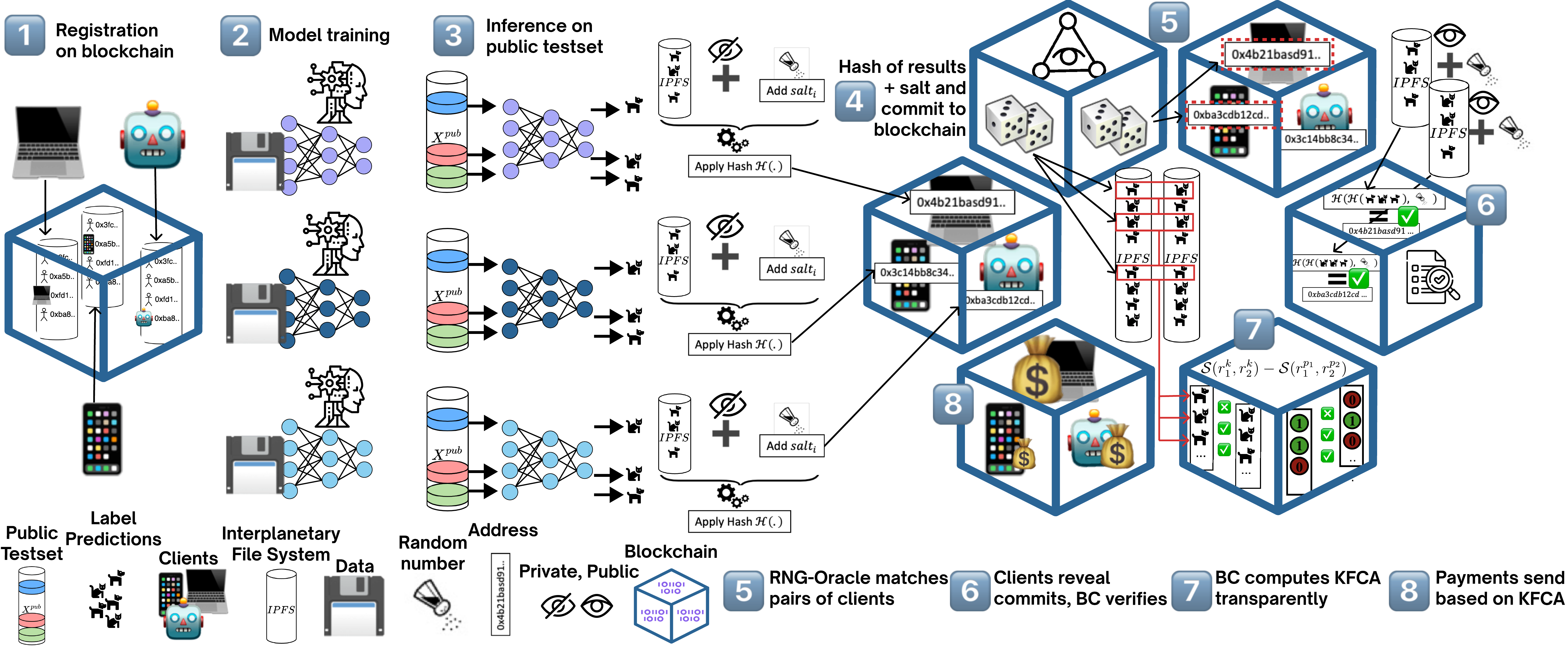}
    \caption{Decentralized and Incentivized Federated Learning: Application of KFCA-D on blockchain.}
    \label{fig:KFCA_X_Blockchain}
\end{figure*}

\subsection{KFCA-QP and KFCA-D Algorithms}
\label{app:algorithms}
In communication round \(t\), let \(M\) denote the set of task indices used for scoring (e.g., public test-set indices for KFCA-D, or parameter-coordinate indices for KFCA-QP). Following the multi-task peer prediction (MTPP) mechanism in Definition~\ref{def:MTPP}, we sample disjoint subsets \(M_b \subset M\) (bonus tasks) and \(M_1, M_2 \subset M \setminus M_b\) (penalty task sets for the two compared clients).

Let \(r_{i,t}^k \in [L]\) denote client \(i\)'s reported label for task \(k \in M\) in round \(t\). Under KFCA, the scoring rule is \(\mathcal{S}_{\mathrm{KFCA}}(r_1,r_2)=\mathds{1}\{r_1=r_2\}\) (Definition~\ref{def:KFCA}). For a given client \(i\), we sample \(P\) peers (denoted \(\mathrm{Peers}(i)\)) and define the final reward as the average KFCA-MTPP payment across peers and bonus tasks:
\[
Q_t^i
= \frac{1}{P\,|M_b|}
\sum_{j \in \mathrm{Peers}(i)}
\sum_{k \in M_b}
\Big(\mathds{1}\{r_{i,t}^k=r_{j,t}^k\} - \mathds{1}\{r_{i,t}^{p_1}=r_{j,t}^{p_2}\}\Big),
\]
where for each \(k \in M_b\) we draw \(p_1 \sim \mathrm{Unif}(M_1)\) and \(p_2 \sim \mathrm{Unif}(M_2)\). Algorithm~\ref{alg:kfca-fl_body} implements this computation.

\begin{algorithm}[H]
\caption{Calculate Final Reward for Client $i$ at Round $t$}
\label{alg:kfca-fl_body}
\begin{algorithmic}[1]
\REQUIRE Task-index set \(M\); disjoint subsets \(M_b \subset M\), \(M_1, M_2 \subset M \setminus M_b\); round \(t\); client set \(N\); target client \(i \in N\); reports \(\{r_{u,t}^k\}_{u \in N,\, k \in M}\); number of sampled peers \(P\)
\ENSURE Final reward/payment of client \(i\) in round \(t\): \(Q_t^i\)

\STATE \(\mathrm{Peers}(i) \leftarrow\) sample \(P\) peers uniformly from \(N \setminus \{i\}\)
\STATE \(Q_t^i \leftarrow 0\)

\FOR{each peer \(j \in \mathrm{Peers}(i)\)}
    \FOR{each bonus task \(k \in M_b\)}
        \STATE \(p_1 \leftarrow\) sample uniformly from \(M_1\); \(p_2 \leftarrow\) sample uniformly from \(M_2\)
        \STATE \(Q_t^i \leftarrow Q_t^i + \mathds{1}\{r_{i,t}^k = r_{j,t}^k\} - \mathds{1}\{r_{i,t}^{p_1} = r_{j,t}^{p_2}\}\)
    \ENDFOR
\ENDFOR
\STATE \(Q_t^i \leftarrow \frac{1}{P\,|M_b|} Q_t^i\)
\end{algorithmic}
\end{algorithm}

\begin{algorithm}[H]
\caption{KFCA-FedAvg Algorithm}
\label{alg:kfca-fedavg}
\begin{algorithmic}[1]
\REQUIRE Client set \(N\); number of rounds \(T\); initial global model \(\theta^0\); method \(\in\{\mathrm{KFCA\text{-}D},\mathrm{KFCA\text{-}QP}\}\); task-index set \(M\); number of sampled peers \(P\)
\ENSURE Global model $\theta^*$

\FOR{\(t = 1\) to \(T\)}
    \FOR{each client \(i \in N\)}
        \STATE \(\theta_{i,t} \leftarrow\) train local model starting from \(\theta^{t-1}\)
        \IF{method is KFCA-D}
            \STATE \(r_{i,t}^k \leftarrow\) predicted label of \(\theta_{i,t}\) on public test example \(k\), for all \(k \in M\)
        \ELSIF{method is KFCA-QP}
            \STATE \(\Delta\theta_{i,t} \leftarrow\) local update (e.g., model/adapter update) produced in round \(t\)
            \STATE \(r_{i,t}^k \leftarrow \mathrm{sign}(\Delta\theta_{i,t}[k])\), for all \(k \in M\)
        \ENDIF
    \ENDFOR
    \STATE Sample disjoint subsets \(M_b \subset M\) and \(M_1, M_2 \subset M \setminus M_b\)
    \FOR{each client \(i \in N\)}
        \STATE \(Q_t^i \leftarrow\) Algorithm~\ref{alg:kfca-fl_body}(\(M, M_b, M_1, M_2, t, N, i, \{r_{u,t}^k\}, P\))
    \ENDFOR
    \STATE \(\theta^{t} \leftarrow \mathrm{FedAvg}(\{\theta_{i,t}\}_{i \in N})\)
\ENDFOR
\STATE \(\theta^* \leftarrow \theta^{T}\)
\end{algorithmic}
\end{algorithm}



\end{document}